%% file: main.tex
\input{_constants}
\arxiv

\pdfoutput=1
\documentclass[10pt,twocolumn,letterpaper]{article}
\input{cvpr_header}
\begin{document}

\title{\paperTitle}
\author{\authorBlock}
\begin{strip}
  \vspace{-1.5cm}
  \maketitle
  \vspace{-0.5cm}
  \centerline{
  \footnotesize
    \setlength{\tabcolsep}{0pt}
    \includegraphics[width=\textwidth]{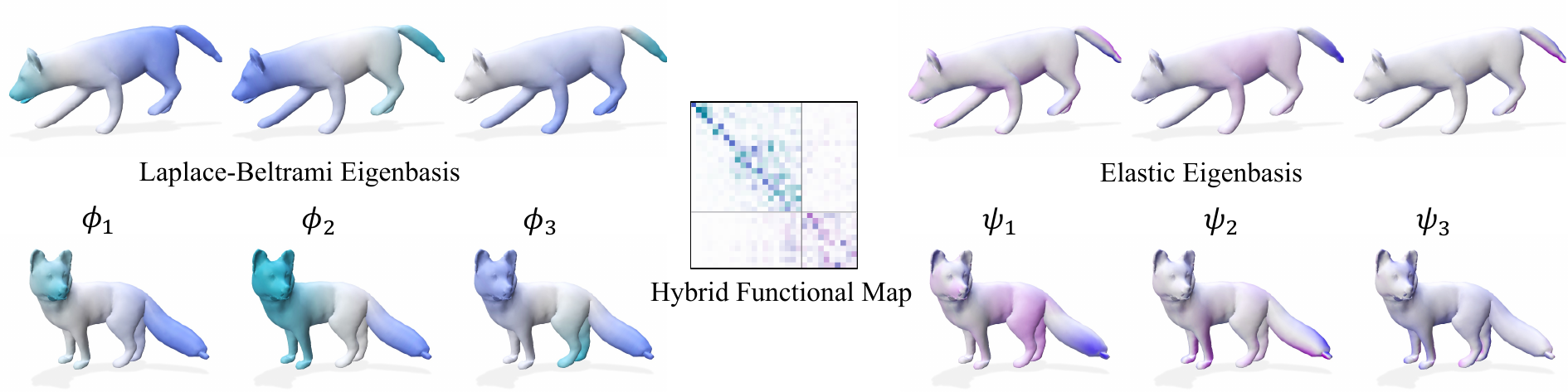}
    \vspace{-0.2cm}
  }
\captionof{figure}{
We propose a novel approach of hybridizing the eigenbases originating from different operators for mapping between function spaces in deformable shape correspondence. 
While the Laplace-Beltrami operator (LBO) eigenbasis is robust to coarse isometric deformations, it fails to encapsulate extrinsic characteristics between shapes. 
In contrast, elastic basis functions~\cite{hartwig_elastic_2023} align with high curvature details but lack the robustness for coarse isometric matching.
The proposed hybrid basis can be used as a drop-in replacement for the LBO basis functions in modern functional map pipelines, improving performance in near-isometric, non-isometric, and topologically noisy settings.
}
\label{fig:teaser}
\end{strip}

{\let\thefootnote\relax\footnote{{
\vspace{-0.2cm}
\parbox{\linewidth}{%
$^*$Equal Contribution\\
\raisebox{-0.2ex}{\Letter} \ : {\tt \scriptsize \{lennart.bastian,yizheng.xie\}@tum.de}
}}} 

\input{00_abstract}
\input{01_intro}
\input{02_related}

\input{03_method}
\input{04_experimental}

\input{05_ablations}

\input{10_conclusion}

{\small
\bibliographystyle{ieeenat_fullname}
\bibliography{11_references}
}

\clearpage
\appendix
\input{12_appendix}

\end{document}

%% file: _constants.tex
\def\paperID{3014}
\def\confName{CVPR}
\def\confYear{2024}

\def\paperTitle{Hybrid Functional Maps for Crease-Aware Non-Isometric Shape Matching}

\def\authorBlock{
    Lennart Bastian\textsuperscript{1*} \qquad
    Yizheng Xie\textsuperscript{1*} \qquad
    Nassir Navab\textsuperscript{1} \qquad
    Zorah Lähner\textsuperscript{2,3,4} \\
    \\
    \textsuperscript{1}\, Technical University of Munich, Germany \quad
    \textsuperscript{2}\, University of Siegen, Germany \\
    \textsuperscript{3}\, University of Bonn, Germany  \quad
    \textsuperscript{4}\, Lamarr Institute, Germany \\ 
}

\newif\ifreview 
\newif\ifarxiv \newcommand{\arxiv}{\arxivtrue}
\newif\ifcamera 
\newif\ifrebuttal 

%% file: cvpr_header.tex
\ifreview \usepackage[review]{cvpr} \fi
\ifarxiv \usepackage[pagenumbers]{cvpr} \fi
\ifrebuttal \usepackage[rebuttal]{cvpr} \fi
\ifcamera \usepackage{cvpr} \fi

\input{_macros}  %

\usepackage{xr-hyper}

\makeatletter
\newcommand*{\addFileDependency}[1]{
  \typeout{(#1)}
  \@addtofilelist{#1}
  \IfFileExists{#1}{}{\typeout{No file #1.}}
}

\makeatother

\definecolor{cvprblue}{rgb}{0.21,0.49,0.74}
\usepackage[pagebackref,breaklinks,colorlinks,citecolor=cvprblue]{hyperref}
\usepackage[capitalize]{cleveref}
\crefname{section}{Sec.}{Secs.}
\crefname{table}{Table}{Tables}
\crefname{figure}{Fig.}{Figs.}

%% file: _macros.tex
\usepackage{graphicx}	
\usepackage{amsmath}	
\usepackage{amssymb}	
\usepackage{booktabs}
\usepackage{times}
\usepackage{microtype}
\usepackage{epsfig}
\usepackage[table,xcdraw,dvipsnames]{xcolor}
\usepackage{caption}
\usepackage{float}
\usepackage{placeins}
\usepackage{color, colortbl}
\usepackage{stfloats}
\usepackage{enumitem}
\usepackage{tabularx}
\usepackage{xstring}
\usepackage{multirow}
\usepackage{xspace}
\usepackage{url}
\usepackage{subcaption}
\usepackage{xcolor}
\usepackage[hang,flushmargin]{footmisc}

\ifcamera \usepackage[accsupp]{axessibility} \fi

\ifarxiv  \fi

\newcommand{\R}[1]{{%
    \textbf{%
        \ifstrequal{#1}{1}{\textcolor{red}{R#1}}{%
        \ifstrequal{#1}{2}{\textcolor{blue}{R#1}}{%
        \ifstrequal{#1}{3}{\textcolor{magenta}{R#1}}{%
        \ifstrequal{#1}{4}{\textcolor{teal}{R#1}}{%
                           \textcolor{cyan}{R#1}%
        }}}}%
    }%
}}

\usepackage{comment}
\usepackage{amsthm}
\newtheorem{theorem}{Theorem}[section]
\newtheorem{proposition}[theorem]{Proposition}
\newtheorem{lemma}[theorem]{Lemma}

\usepackage{color}
\usepackage{soul}

\makeatletter
\renewcommand\footnoterule{%
  \kern-50pt%
  \hrule\@width.4\columnwidth
  \kern2.6pt%
}
\makeatother

\definecolor{yizheng}{rgb}{0.9,0.,0.5}

\definecolor{zorah}{rgb}{0.,0.6,0.3}

\definecolor{ForestGreen}{rgb}{0.13, 0.55, 0.13}
\usepackage{pifont}%
\newcommand{\cmark}{{\color{ForestGreen} \ding{51}}}%
\newcommand{\xmark}{{\color{red} \ding{55}}}%

\definecolor{TRColor}{RGB}{240, 240, 240}
\newcommand{\hlgray}[1]{{\sethlcolor{TRColor}\hl{#1}}}

\usepackage{cuted}

\AtEndPreamble{
    \crefname{theorem}{Thm.}{Thms.}
    \crefname{lemma}{Lemma}{Lemmas}
}
\usepackage[misc]{ifsym}

\definecolor{darkgreen}{HTML}{228C21}
\renewcommand{\top}{T}
\usepackage{makecell}%

%% file: 00_abstract.tex
\begin{abstract}
\vspace{-0.5cm}
Non-isometric shape correspondence remains a fundamental challenge in computer vision. 
Traditional methods using Laplace-Beltrami operator (LBO) eigenmodes face limitations in characterizing high-frequency extrinsic shape changes like bending and creases. 
We propose a novel approach of combining the non-orthogonal extrinsic basis of eigenfunctions of the elastic thin-shell hessian with the intrinsic ones of the LBO, creating a hybrid spectral space in which we construct functional maps. 
To this end, we present a theoretical framework to effectively integrate non-orthogonal basis functions into descriptor- and learning-based functional map methods.
Our approach can be incorporated easily into existing functional map pipelines across varying applications and can handle complex deformations beyond isometries.
We show extensive evaluations across various supervised and unsupervised settings and demonstrate significant improvements. 
Notably, our approach achieves up to 15\% better mean geodesic error for non-isometric correspondence settings and up to 45\% improvement in scenarios with topological noise.
Code is available at: \href{https://hybridfmaps.github.io/}{https://hybridfmaps.github.io/}
\end{abstract}

%% file: 01_intro.tex
\section{Introduction}
\label{sec:intro}

Establishing dense correspondences between 3D shapes is a cornerstone for numerous computer vision and graphics tasks such as object recognition, character animation, and texture transfer.
The complexity of this task varies significantly depending on the nature of the transformation a shape undergoes.
Many classic correspondence methods leverage that rigid transformations can be represented in six degrees of freedom in $\mathbf{R}^3$ and preserve the Euclidean distance between pairs of points.
Iterative closest point (ICP)~\cite{arun1987icp}, and its variations~\cite{li2008nonrigidicp,pomerleau2013icpvariants}, which alternate between transformation and correspondence estimation, are very popular due to their simplicity. 
In this setting, \emph{local} extrinsic surface properties in the embedding space stay invariant under rigid transformations such that they can be used as features during optimization, for example, the change of normals.
For the wider class of isometric deformations (w.r.t. the geodesic distance), the relative embedding of the shape can change significantly, and Euclidean distances between points may not be preserved.
In this class, only \emph{intrinsic} properties -- those that do not depend on a specific embedding of the surface
-- stay invariant, and the correspondence problem becomes much harder due to the quadratic size of the solution space.
For example, solving a quadratic assignment problem preserving geodesic distances \cite{kezurer2015tight} or heat kernel \cite{vestner2017kernel} is \emph{intrinsic} by nature%
, but it is also an NP-hard problem. 

In this context, spectral shape analysis, a generalization of Fourier analysis to Riemannian manifolds, has emerged as a powerful tool for non-rigid correspondence by leveraging intrinsic shape structure.
One popular method that takes advantage of this tool is functional maps, introduced by Ovsjanikov et al.~\cite{ovsjanikov_functional_2012}, which synchronizes the eigenfunctions of the Laplace-Beltrami operator (LBO) through a low-dimensional linear change of basis.
Numerous adaptations have led to advances in shape correspondence in recent years, for example, in both the learned supervised~\cite{litany_deep_2017,donati_deep_2020} and unsupervised settings~\cite{roufosse_unsupervised_2019,li_learning_2022,cao_self-supervised_2023,cao_unsupervised_2023,sun_spatially_2023}; and
while other basis choices have been proposed~\cite{neumann2014compressed,panine2022landmark,hartwig_elastic_2023}, almost all of these methods use the eigenfunctions of the LBO to span the to-be-mapped function spaces. 
One reason is that the LBO has been extensively studied, and the behavior of its eigenfunctions is well understood.
For instance, the LBO's eigenfunctions have a relatively consistent ordering and general invariance under isometric deformations.
These understandings have been leveraged for efficient regularization~\cite{rodola_partial_2015} and coarse-to-fine optimization~\cite{melzi_zoomout_2019,eisenberger_smooth_2020}.
Other basis sets have been studied and shown to be effective in specific cases~\cite{panine2022landmark,colombo2022pcgau}, but none are so generally applicable and flexible as the LBO eigenfunctions.

A known weakness of the LBO basis, which at the same time comes from its biggest strength, is the reduction to low-frequency information.
This leads to efficient optimization and robustness to noise but also inaccuracy in small details (see \cref{fig:motivation_vertex}). 
To counter this challenge, Hartwig et al.~\cite{hartwig_elastic_2023} proposed to utilize a basis derived from the spectral decomposition of an elastic thin-shell energy for functional mapping.
These bases are particularly suitable for aligning extrinsic features of non-isometric deformations, for example, bending and creases~\cite{hartwig_elastic_2023}.
However, due to the non-orthogonality of these basis functions, careful mathematical treatment is required to construct the appropriate optimization problem.
Furthermore, the elastic basis functions do not exhibit the isometric invariance and robustness of the LBO basis functions, limiting their applicability (see~\cref{sec:ablation}).

To address the shortcomings of the bases on their own, we propose to estimate functional maps in a \emph{hybrid} basis representation.
We achieve this by constructing a joint vector space between the LBO basis functions and those of the thin shell hessian energy~\cite{wirth_continuum_2011, hartwig_elastic_2023}.
We demonstrate that combining intrinsic and extrinsic features in this manner provides several advantages for both near-isometric and non-isometric shape-matching problems, promoting robust functional maps that can represent fine creases in the shapes as well as large topological changes.
Due to the principled nature of our approach, the combined basis representation can be used in place of pure LBO basis functions in many functional map-based methods.
We demonstrate this on several of the most strongly performing axiomatic and learning-based pipelines, leading to considerable performance improvements on various challenging shape-matching datasets.

\noindent\textbf{Contributions. } Our contributions are as follows:
\begin{itemize}
    \item We introduce a theoretically grounded framework to estimate functional maps between non-orthogonal basis sets using descriptor-based linear systems, a foundational element of nearly all functional map-based learning methods.%
    \item We propose a hybrid framework for estimating functional maps that leverage the strengths of basis functions originating from different operators.
    We employ this framework to construct functional maps robust to coarse deformations and topological variations while capturing fine extrinsic details on the shape surface.
    \item We conduct an extensive experimental validation establishing a strong case for the proposed hybrid mapping framework in various challenging problem settings, achieving notable improvements upon state-of-the-art methods for deformable correspondence estimation.
\end{itemize}

%% file: 02_related.tex
\section{Related Work}
\label{sec:related}

Shape understanding has been studied extensively; a comprehensive background is beyond the scope of this work.
We refer the reader to one of several recent surveys~\cite{sahillioglu_recent_2020,deng_survey_2022}. 
This section provides an overview of the works most closely related to ours.

\noindent\textbf{Intrinsic-Extrinsic Methods. }
Both intrinsic and extrinsic approaches have advantages and disadvantages, and an optimal method probably uses both. 
Several works combining the functional maps framework with extrinsic features exist, for example, with SHOT descriptors~\cite{salti_shot_2014}, including surface orientation information~\cite{ren_continuous_2018,donati_complex_2021}, anisotropic information~\cite{andreux2014anisotropic}, or spatial smoothness of the point map~\cite{sun_spatially_2023}.
SmoothShells \cite{eisenberger_smooth_2020} uses extrinsic information as a deformation field, aligning the surfaces in a coarse-to-fine approach guided by the frequency information of the LBO eigenfunctions.
These approaches still use the purely intrinsic LBO eigenfunctions to define the functional maps basis, adding extrinsic information through regularization or additional steps.

\noindent\textbf{Functional Maps}.
The functional map framework proposed in~\cite{ovsjanikov_functional_2012} uses the eigenfunctions of the LBO to pose the correspondence problem as a low-dimensional linear system by rephrasing it as a correspondence of basis functions instead of vertices.
The frequency-ordering of the LBO eigenfunctions, as well as their invariance to isometries, allow them to span a comparable but expressive space of smooth functions, which can be efficiently matched by using point descriptors, for example HKS~\cite{sun_concise_2009}, WKS~\cite{aubry_wave_2011} or SHOT~\cite{salti_shot_2014}. 

Follow-up work has been proposed to improve the correspondence quality~\cite{pai_fast_nodate,melzi_zoomout_2019}, extend it to more general settings~\cite{rodola_partial_2015,huang_limit_2019}, and learn to generate optimal descriptors~\cite{litany_deep_2017,halimi_unsupervised_2019,sharp_diffusionnet_2022}.
These methods are particularly powerful as they exploit the structure of the geometric manifolds through the functional correspondence of eigenfunctions on the shapes but still incorporate a learned descriptor to more accurately represent nuances in the shape surface topology.
Unsupervised learning-based approaches have been proven highly effective in recent years~\cite{roufosse_unsupervised_2019,attaiki_dpfm_2021,li_learning_2022,cao_self-supervised_2023,cao_unsupervised_2022,cao_unsupervised_2023}, even surpassing the performance of supervised methods.
Such approaches have not only succeeded on a wide range of computer vision benchmarks but have recently proven effective in the medical domain~\cite{magnet_assessing_2023,bastian_s3m_2023,cao_self-supervised_2023,bastian2023localization}.

\noindent\textbf{Basis Functions}. 
Many improvements have been proposed for the functional map framework, but most methods still use the Laplace-Beltrami eigenfunctions as the underlying basis. 
Despite this, other basis types have been proposed for shape analysis, for example, the L1-regularized spectral basis~\cite{neumann2014compressed}, the landmark-adapted basis~\cite{panine2022landmark}, a basis derived from gaussians~\cite{colombo2022pcgau}, or localized manifold harmonics \cite{melzi2017localized}.
The latter proposed to ``mix" a localized basis with the normal LBO eigenfunctions. 
DUO-FMNet~\cite{donati_deep_2022} proposes calculating an additional functional map for the complex-valued connection Laplacian basis.
However, the basis functions in both cases are orthogonal and purely intrinsic.
Another approach is to learn the optimal basis set for functional maps \cite{marin_correspondence_2020,huang_multiway_2022,siddiqi2023network}, but these tend to not generalize to new applications and, thus, cannot be used out of the box.
Various other extrinsic bases have been proposed \cite{liu2017dirac,aflalo2013scale,wang2018steklov}, but none of these have been demonstrated suitable for functional correspondence.

Recently, Hartwig et al. introduced an elastic basis based on the eigendecomposition of the Hessian of the thin-shell deformation energy for functional maps~\cite{hartwig_elastic_2023}. 
While it preserves some desirable properties of the LBO (like frequency information) and is better suited for detail alignment, our results indicate it does not perform well in learned functional map-based pipelines (c.f. ~\cref{sec:ablation}).
In this work, we analyze the reasons for this and propose a novel way to preserve the advantages of the elastic basis while joining it with the performance of LBO-based approaches.

%% file: 03_method.tex
\section{Background: Functional Maps}
\label{sec:background}

Functional maps \cite{ovsjanikov_functional_2012} offer a compelling framework for shape matching by abstracting point-to-point correspondences $S_1 \to S_2$ to a functional representation between function spaces on manifolds $\mathcal{F}(S_1) \to \mathcal{F}(S_2)$.
This paradigm simplifies the map optimization problem to a linear and compact (low-rank) form, enabling additional regularization.

Until now, the Laplace-Beltrami eigenfunctions have been used almost exclusively as the basis to span the to-be-matched function spaces due to their desirable properties, for example, orthogonality, isometry invariance, and allowing a significant dimensionality reduction.
In~\cref{sub:bg:nonorth}, we will study the more general setting of computing functional maps for non-orthogonal basis sets, an extension of the non-orthogonal ZoomOut~\cite{melzi_zoomout_2019} that has been proposed in~\cite{hartwig_elastic_2023}.
But first, we introduce the default functional map framework.

\noindent \textbf{Spectral Decomposition.} 
A positive semidefinite (p.s.d) linear operator $\mathcal{T}$ (in most cases the LBO, $\Delta$)
is computed on the mesh representation of each shape, followed by solving the generalized eigenvalue problem:
\begin{equation}
\label{eq:generalized_eigenvalue}
\mathcal{T} \phi_i = \lambda_i M \phi_i.
\end{equation}

Ordered by eigenvalues, the first $k$ eigenfunctions $\Phi_k$ can be used as a truncated basis for each shape.
As both $\mathcal{T}$ and the mass matrix of lumped area elements for each shape $M$ are p.s.d., the eigenfunctions are orthogonal w.r.t the norm induced on the vector space by $M$: $\Phi^\top_k M \Phi_k = I$.

\noindent\textbf{Functional Map Estimation.}
Given two point descriptors functions $D_1 \in \mathcal{F}(S_1), D_2 \in \mathcal{F}(S_2)$ which are known to be corresponding, the functional map between two basis sets can be computed via a least-squares problem.
Let $D_{\Phi_i}:= \Phi^\dagger_i D_i$ denote the descriptor functions projected into the LBO eigenfunctions $\Phi_i$ using the Moore-Penrose pseudo inverse $\Phi^\dagger_i$.
We can then compute an optimal functional map by solving the following optimization problem ~\cite{ovsjanikov_functional_2012,donati_deep_2020}:
\begin{align}
\label{eq:linear_lsq}
C^* = \arg \min_C E(C) &= E_{\text{data}}(C) + \lambda E_{\text{reg}}(C) \\ \nonumber
E_{\text{data}}(C) &= \| C D_{\Phi_1} - D_{\Phi_2} \|_F^2 \\ \nonumber
E_{\text{reg}}(C) &= \| C \Lambda_1  -  \Lambda_2 C \|_F^2 \nonumber
\end{align}

\noindent where $\Lambda_1$ a diagonal matrix of the eigenvalues of $\mathcal{T}$~\cite{donati_deep_2020} or the resolvant~\cite{ren_structured_2019}.
This energy can be solved in closed form row-by-row with $k$ least squares problems when defined in the Frobenius norm~\cite{donati_deep_2020}.

Learned features have proven robust for a wide variety of surface representations.
Unless mentioned otherwise, we use deep features from  DiffusionNet~\cite{sharp_diffusionnet_2022} and denote these as $D_i \in \mathbb{R}^{n_i \times d}$ for shapes $S_1$ and $S_2$.

\noindent \textbf{Map Regularization.}
The estimated map can be interpreted as a change of basis between shapes. 
In case of an underdetermined linear system in \cref{eq:linear_lsq} or noisy descriptor function, 
$C$ can be further regularized with losses that promote orthogonality, bijectivity, isometry, or additional pointwise descriptor preservation \cite{roufosse_unsupervised_2019,donati_deep_2020,cao_unsupervised_2023}.
If the regularizer is in a simple quadratic form, it can be backpropagated through and used to train the descriptor functions.

\subsection{Non-Orthogonal Basis Functions} \label{sub:bg:nonorth}
Wirth et al.~\cite{wirth_continuum_2011} originally proposed an elastic thin-shell energy for spectral analysis.
Hartwig et al.~\cite{hartwig_elastic_2023} then recently demonstrated how the spectral decomposition of this elastic deformation energy can be used for functional mapping despite being non-orthogonal~\cite{hartwig_elastic_2023}.
The elastic energy $\mathcal{W}_S[\mathbf{f}]$ consists of a membrane contribution $\mathcal{W}_{\text{mem}}$, which measures the local distortion of the surface, and bending energy $\mathcal{W}_{\text{bend}}$ encapsulating curvature (c.f. appendix for a complete definition).
By construction, the semi-positive definite hessian of the elastic deformation energy can be decomposed at the identity as in~\cref{eq:generalized_eigenvalue}, yielding a set of eigenfunctions.
These vector-valued eigenmodes are suitable for functional mapping after projection onto the vertex-wise normals of the surface and selecting the first $k$ non-orthogonal basis functions $\Psi = [\psi_1,...,\psi_k ] \in \mathbf{R}^{n, k}$~\cite{hartwig_elastic_2023}.

Much of the simplicity of the functional maps framework can be attributed to the orthogonality of the basis functions w.r.t the mass matrices on each shape.
The mass matrix accounts for the anisotropic metric on the non-Euclidean shape manifolds which must be observed for common operators such as the inner product $\langle \cdot, \cdot \rangle_M$ and norm $||\cdot||_M$.
The reduced mass representation $M_{k} = \Psi^\top M \Psi \in \mathbf{R}^{k \times k}$ can be used to construct a metric in the spectral space of each shape.
Notably, these operations reduce to the standard inner product when represented in the orthogonal LBO basis.

However, this is not the case for non-orthogonal basis functions, and careful treatment must be taken to avoid neglecting the anisotropic metric.
Hartwig et al.~\cite{hartwig_elastic_2023} derive the necessary operations, such as the orthogonal projector, and reformulate optimization problems to use the elastic basis in the ZoomOut~\cite{melzi_zoomout_2019} framework for functional map refinement.
For a thorough treatment of these fundamental definitions, we refer the reader to the relevant literature~\cite{wirth_continuum_2011,hartwig_elastic_2023}.
Our method requires several additional operations and losses to utilize the elastic basis in a learned setting, including the formulation in \cref{eq:linear_lsq}, which we will derive in \cref{sec:method}.

\input{content/pck_motivational}

\input{content/vertex_map_motivational}

\section{Method: A Hybrid Approach}
\label{sec:method}

The LBO eigenbasis is the predominant choice in functional map-based~\cite{ovsjanikov_functional_2012} approaches due to their robustness and invariance to isometric deformations, but they tend to struggle with aligning high-frequency details.
On the other hand, the recently proposed elastic basis functions have proven effective at representing extrinsic creases and bending~\cite{hartwig_elastic_2023} (see ~\cref{fig:motivation_pck}).
However, we observed that naively replacing the LBO basis with the elastic basis does not always improve performance, particularly in learning-based frameworks (see ~\cref{table:ablation}).

To overcome the deficiencies of both basis choices, we propose constructing functional maps between \emph{hybrid spaces consisting of the LBO and elastic basis functions}.
This attains the best of both worlds: a stable, isometric functional map at low frequencies and sensitivity to extrinsic creases and high-curvature details.
To achieve this, we generalize the deep functional maps framework outlined in~\cref{sec:background} to non-orthogonal basis functions in \cref{sec:hilbert}, then introduce the hybrid functional map estimation in~\cref{sec:hybrid_basis}, and discuss necessary adjustments for learning pipelines in~\cref{sec:learning_hybrid}.

\subsection{Generalization to the Hilbert-Schmidt Norm}\label{sec:hilbert}

In this section, we will generalize \cref{eq:linear_lsq} to functional maps between non-orthogonal basis sets, for example, the elastic basis~\cite{hartwig_elastic_2023}.
For an orthogonal basis, \cref{eq:linear_lsq} can be written with the Frobenius norm in spectral space. 
For non-orthogonal basis functions, this requires using an inner product induced by the mass matrices on each shape ~\cite{hartwig_elastic_2023}; norms to measure distances in each Hilbert space or the magnitude of linear operators must be scaled similarly.

\textbf{Data Term.} The original formulation of \cref{eq:linear_lsq} takes the difference of the descriptors $D_1, D_2$ as functions on the surface using the $S_2$ inner product; this reduces to the standard inner product in spectral space for the LBO eigenfunctions.
For non-orthogonal basis sets, the spectral space is a Hilbert space equipped with an inner product induced by the reduced mass matrix $M_{k,2} = \Psi_2^\top M_2 \Psi_2$.
The data term then reads:
\begin{lemma}
\label{lemma:edata_induced_norm}
The descriptor preservation term $E_\text{data}$ can be represented in the norm induced by $M_{k,2}$ as:
\begin{equation}
\Vert CD_{\Psi_1} - D_{\Psi_2} \Vert_{M_{k,2}} = {\Vert \sqrt{M_{k,2}} (CD_{\Psi_1} - D_{\Psi_2})  \Vert}_F
\end{equation}
\end{lemma}

\noindent %
We include a derivation in the appendix for completeness.

\textbf{Regularizer.} Next, we derive $E_{\text{reg}}$ which ensures the functional map $C$ commutes with the diagonal matrix of eigenvalues $\Lambda_i$ of the respective linear operator ~\cite{ovsjanikov_functional_2012, donati_deep_2020} or its resolvant \cite{ren_structured_2019}.
A key to functional map formulation of Hartwig et al.~\cite{hartwig_elastic_2023} is the use of the Hilbert-Schmidt norm for linear operators between Hilbert spaces, as it considers the geometry on both the domain \textit{and} range of the operator as opposed to the Frobenius norm.
We note that the term $E_{\text{reg}}$ measures the magnitude of the operator $(C \Lambda_1  -  \Lambda_2 C): \mathcal{F}(S_1) \rightarrow \mathcal{F}(S_2)$, and should therefore take into account the anisotropic metrics on each space.

\input{content/method.tex}

\begin{proposition}
\label{thm:ereg_hs_norm}
The regularization term $E_\text{reg}$ can be formulated in the Hilbert-Schmidt norm as: 
\begin{align}
\lVert C \Lambda_1  -  \Lambda_2 C \rVert_{HS} = {\lVert \sqrt{M_{k,2}}(C \Lambda_1  -  \Lambda_2 C)\sqrt{M_{k,1}^{-1}} \rVert}_F \nonumber
\end{align}

This problem can be expanded to a $k^2 \times k^2$ linear system:

\begin{align}
{\lVert ((\Lambda_1 \sqrt{M_{k,1}^{-1}}) \otimes \sqrt{M_{k,2}} - \sqrt{M_{k,1}^{-1}} \otimes (\sqrt{M_{k,2}} \Lambda_2)) \vec{C} \rVert}_2 \nonumber
\end{align}

\noindent with the Kroneker product $\otimes$, and $\vec{C} = vec(C)$ the column stacked vectorization of $C$, and using Lemma~\ref{lemma:edata_induced_norm}. This system can be solved in closed form.
\end{proposition}

\begin{proof}
The first statement follows from the definition of the HS-norm, using the cyclicity of the trace and equivalence with the scaled Frobenius norm. 
A detailed discussion regarding how to reformulate this optimization problem in the expanded form can be found in the appendix.
\end{proof}

It was previously shown that the formulation of $E(C)$ in the Frobenius norm admits a closed-form solution~\cite{donati_deep_2020,ren_structured_2019}.
This is crucial to the deep functional maps pipeline; we found that solving iteratively with a differentiable convex optimizer proves prohibitively expensive when solved to sufficient accuracy for the outer SGD iteration to converge.
We therefore solve~\cref{eq:linear_lsq} in the expanded form to consider the anisotropic metrics on $\mathcal{F}(S_1)$ and $\mathcal{F}(S_2)$.
However, the expanded $k^2 \times k^2$ system becomes prohibitively large at $k=200$, the spectral resolution typically used in advanced functional maps pipelines.
In the next section, we show that separating the functional map optimization in~\cref{eq:linear_lsq} into two problems under mild assumptions effectively resolves this issue and enables proper regularization for practical applications.

\input{content/main_results}

\subsection{Hybrid Functional Map Estimation}
\label{sec:hybrid_basis}
Our experiments suggest that although the elastic basis performs sub-optimally compared to the LB basis in deep-learning settings, it achieves a higher percentage of matches at a low geodesic error threshold, suggesting superior alignment of fine details and creases (see~\cref{fig:motivation_pck}).
Motivated by this, we propose constructing a hybrid basis by combining basis functions from both operators.
Intuitively, the low-frequency LBO eigenfunctions approximate the shape and enable coarse alignment, while the elastic eigenfunctions conform to creases and regions of high curvature.
In this hybrid function space, a functional map $C$  is articulated as a block matrix, with each entry $C^{ij}$ encoding the correspondence between two basis sets (see \cref{fig:teaser}).

\begin{equation}
\label{eq:block_matrix_formulation}
C = \begin{pmatrix}
C^{11} & C^{12} \\
C^{21} & C^{22}
\end{pmatrix}
\end{equation}

$C^{11}$ and $C^{22}$ correspond to intra-basis maps and the off-diagonal blocks $C^{12}$ and $C^{21}$ to inter-basis maps.
The resulting hybrid map can be used directly to obtain dense point-to-point correspondences via nearest neighbor search in the hybrid basis or via map refinement strategies~\cite{melzi_zoomout_2019}.

We observe that, while mutually non-orthogonal, the LBO and elastic eigenbasis exhibit very different behaviors on the shape, and therefore assume that inter-basis maps are undesirable (enforcing $C^{12} = \mathbf{0} \ \text{and} \ C^{21} = \mathbf{0}$).
\cref{eq:block_matrix_formulation} then separates into two optimization problems (one for each basis type), resulting in a block diagonal functional map (see~\cref{fig:framework}).
In the appendix, we support this with mathematical intuition and empirically show that inter-basis matchings adversely affect the convergence of the map.

The combined framework enables the application of the hybrid basis to large functional map systems (e.g. $k=200$) with proper adaptation for a non-orthogonal basis, combining the benefits of different basis types and exceeding the performance of each basis when used individually (see~\cref{fig:motivation_pck}).

\subsection{Learning in a Hybrid Basis}
\label{sec:learning_hybrid}

Deep functional map pipelines regularize the FM obtained from~\cref{eq:linear_lsq} with additional loss functions.
These are described in detail in~\cref{sec:experimental} for each specific pipeline.
Similar to~\cref{sec:hybrid_basis}, we note that each loss can be separated in our block diagonal hybrid formulation (c.f. appendix for details).

The isometric invariance of the LBO loss functions provides a strong supervision signal for unsupervised FM methods.
As the elastic basis functions lack these properties, training from scratch with various architectures leads to suboptimal convergence.
We, therefore, parameterize the optimization in \cref{eq:fm_loss} through linear annealing during training:
\begin{align}
\label{eq:fm_loss}
\mathcal{L}(C) = \mathcal{L}_{\text{LB}}(C) + \mu \mathcal{L}_{\text{Elas}}(C)
\end{align}
Intuitively, this favors coarse isometric matching early on during training with the LBO eigenfunctions and leverages the tendency of the elastic basis to align creases and details later for optimal convergence.
Empirically, we find this achieves superior performance compared to fine-tuning from LBO pre-trained descriptors, which likely converges to local minima near the LBO optimum.

%% file: content/pck_motivational.tex
\begin{figure}[t]
  \centering
  \vspace{-0.1cm}
  \includegraphics[width=0.5\textwidth]{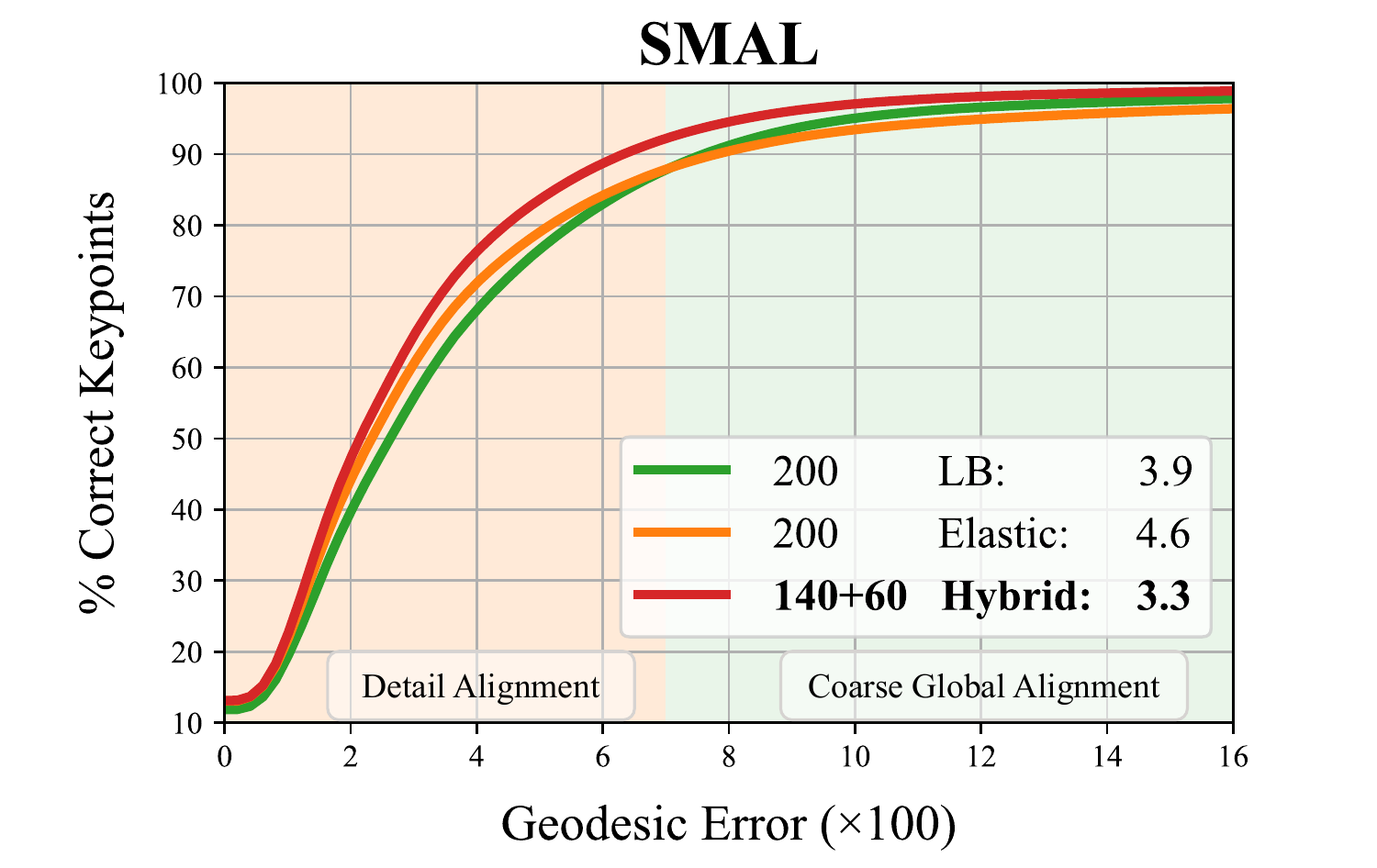}
  \vspace{-0.8cm}
  \caption{A Percentage-Correct-Keypoint ablation between the pure LB basis, pure elastic basis (orthogonalized), and our hybrid approach at the same spectral resolution ($k = 200$). The elastic basis attains better detail alignment than the LB basis but yields inferior overall global correspondences. The proposed hybrid approach achieves the best of both worlds.
  Experiments are conducted with the ULRSSM \cite{cao_unsupervised_2023} framework on SMAL.
  }
  \vspace{-0.5cm}
  \label{fig:motivation_pck}
\end{figure}

%% file: content/vertex_map_motivational.tex
\begin{figure}[t]
  \centering
  \vspace{-0.1cm}
  \includegraphics[width=0.5\textwidth]{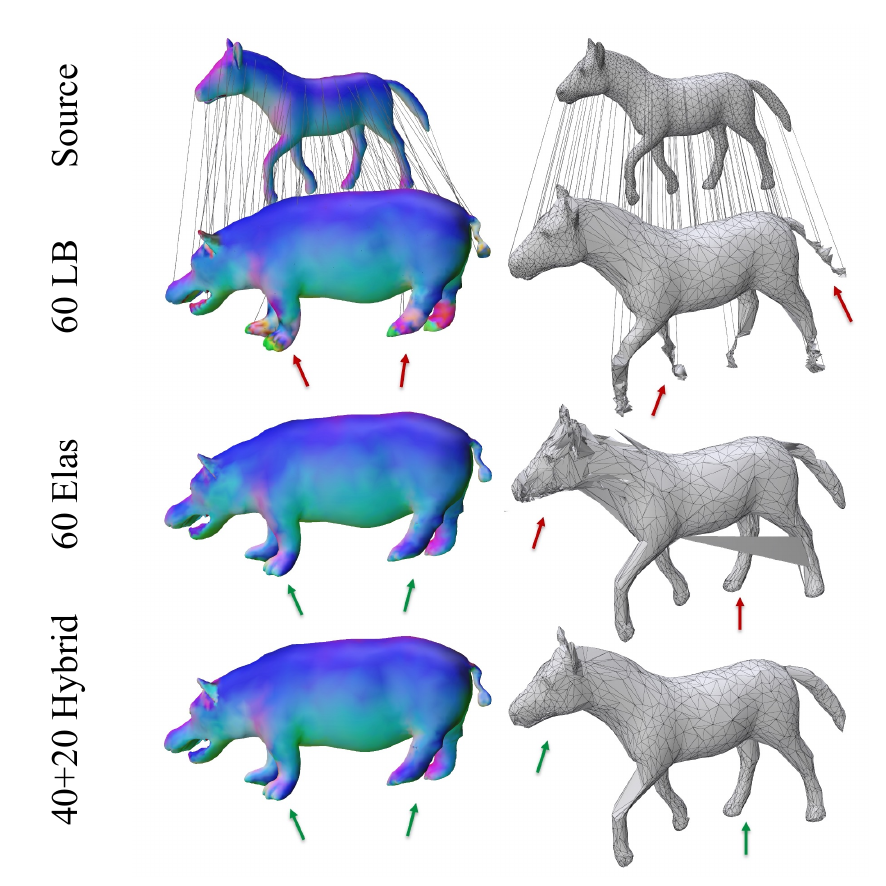}
  \vspace{-0.8cm}
  \caption{Correspondence quality visualized through the transfer of normals (left) and vertex positions (right) from the source to the target shape. We compare the results from LB, Elastic, and Hybrid basis functions by encoding and recovering \textbf{ground truth} correspondences through functional map representations at a spectral resolution of $k = 60$. Additional results at a spectral resolution of $k = 200$ are provided in the supplementary.
  }

  \vspace{-0.5cm}
  \label{fig:motivation_vertex}
\end{figure}

%% file: content/method.tex
\begin{figure*}[t]
  \centering
  \vspace{-0.2cm}
  \includegraphics[width=\linewidth]{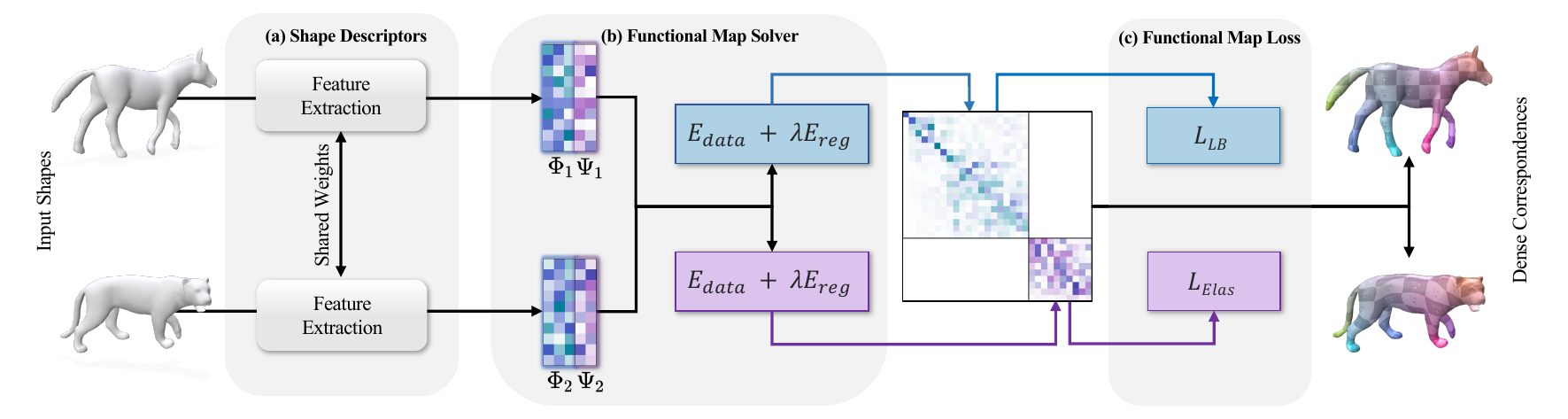}
  \vspace{-0.7cm}
  \caption{\textbf{Hybrid Functional Maps} in a typical pipeline. 
  Features are first extracted from a pair of shapes with a Siamese network (a). 
  They are then projected onto eigenbasis sets from different linear operators (b).
  We then solve for a block diagonal functional map spanning the constructed hybrid function space (b). 
  Additional regularization can be used to impose structure on parts of the hybrid functional map (c). 
  }
  \label{fig:framework}
  \vspace{-0.2cm}
\end{figure*}

%% file: content/main_results.tex
\begin{table*}
    \setlength{\tabcolsep}{7pt}
    \small
    \centering
    \caption{\textbf{Shape correspondence estimation} under various conditions, including isometric, non-isometric, and settings with topological noise. The \hlgray{proposed hybrid approach} yields performance improvements in axiomatic, supervised, and unsupervised settings.\\
    $\dag$ SHREC'19 methods are trained on FAUST and SCAPE as in recent methods~\cite{li_learning_2022,cao_unsupervised_2023}.}
    \vspace{-0.2cm}
    \label{tab:main_results}
    \begin{tabular}{@{}c|l|ccc|ccc|c|@{}}
    \toprule
    
    & \multicolumn{1}{l|}{\multirow{2}{*}{\textbf{Geodesic Error ($\times$100)}}} & \textbf{FAUST} & \textbf{SCAPE} & $\textbf{SHREC'19}^{\dag}$ & \textbf{SMAL} & \multicolumn{2}{c|}{\textbf{DT4D-H}} & \textbf{TOPKIDS} \\
    & \multicolumn{1}{l|}{} & & & & & \textbf{intra-class} & \textbf{inter-class} & \\
    \midrule
    \parbox[t]{2mm}{\multirow{4}{*}{\rotatebox[origin=c]{90}{\textit{Axiomatic}}}}
    & ZoomOut~\cite{melzi_zoomout_2019} & 6.1 & 7.5 & - & 38.4 & 4.0 & 29.0 & 33.7 \\
    & DiscreteOp~\cite{ren_discrete_2021} & 5.6 & 13.1 & - & 38.1 & 3.6 & 27.6 & 35.5 \\
    & Smooth Shells~\cite{eisenberger_smooth_2020} & \textbf{2.5} & 4.2 & - & 30.0 & \textbf{1.2} & 6.4 & 10.8 \\ 
    &\cellcolor{TRColor}{Hybrid Smooth Shells (ours)} & \cellcolor{TRColor}{2.6} & \cellcolor{TRColor}{4.2} & \cellcolor{TRColor}{-} & \cellcolor{TRColor}{\textbf{28.4}} & \cellcolor{TRColor}{1.3} & \cellcolor{TRColor}{\textbf{5.7}} & \cellcolor{TRColor}{\textbf{7.5}} \\
    \midrule
    \parbox[t]{2mm}{\multirow{3}{*}{\rotatebox[origin=c]{90}{\textit{Sup.}}}}
    & FMNet~\cite{litany_deep_2017} & 11.0 & 33.0 & - & 42.0 & 9.6 & 38.0 & - \\
    & GeomFMaps~\cite{donati_deep_2020} & 2.6 & 3.0 & 7.9 & 8.4 & \textbf{1.9} & 4.2 & - \\
    &\cellcolor{TRColor}{Hybrid GeomFMaps (ours)} & \cellcolor{TRColor}{\textbf{2.4}} & \cellcolor{TRColor}{\textbf{2.8}} & \cellcolor{TRColor}{\textbf{5.6}} & \cellcolor{TRColor}{\textbf{7.6}} & \cellcolor{TRColor}{2.2} & \cellcolor{TRColor}{\textbf{4.1}} & \cellcolor{TRColor}{-} \\
    \midrule
    \parbox[t]{2mm}{\multirow{7}{*}{\rotatebox[origin=c]{90}{\textit{Unsupervised}}}}
    & Deep Shells~\cite{eisenberger_deep_2020} & 1.7 & 2.5 & 21.1 & 29.3 & 3.4 & 31.1 & 13.7 \\
    & DUO-FMNet~\cite{donati_deep_2022} & 2.5 & 4.2 & 6.4 & 6.7 & 2.6 & 15.8 & - \\
    & AttentiveFMaps-Fast~\cite{li_learning_2022} & 1.9 & 2.1 & 6.3 & 5.8 & 1.2 & 14.6 & 28.5 \\
    & AttentiveFMaps~\cite{li_learning_2022} & 1.9 & 2.2 & 5.8 & 5.4 & 1.7 & 11.6 & 23.4 \\
    & SSCDFM~\cite{sun_spatially_2023} & 1.7 & 2.6 & 3.8 & 5.4 & 1.2 & 6.1 & - \\
    & ULRSSM~\cite{cao_unsupervised_2023} & 1.6 & 1.9 & 4.6 & 3.9 & \textbf{0.9} & 4.1 & 9.2 \\
     &\cellcolor{TRColor}{Hybrid ULRSSM (ours)} & \cellcolor{TRColor}{\textbf{1.5}} & \cellcolor{TRColor}{\textbf{1.8}} & \cellcolor{TRColor}{\textbf{3.6}} & \cellcolor{TRColor}{\textbf{3.3}} & \cellcolor{TRColor}{1.0} & \cellcolor{TRColor}{\textbf{3.5}} & \cellcolor{TRColor}{\textbf{5.0}} \\
    \bottomrule
    \end{tabular}
    \vspace{-0.2cm}
\end{table*}

%% file: 04_experimental.tex
\section{Experimental Results} 
\label{sec:experimental}
This section provides a summary of the datasets used and our experimental setup. 
We refer to the appendix for a complete description of the datasets, splits, hyperparameters, and reformulation of method-specific losses in the HS-norm. 
We use $k=k_{\text{LB}}+k_{\text{Elas}}$ to signify the total spectral resolution, the number of LBO and elastic basis functions, respectively. 

\subsection{Datasets}
We evaluate our method on several challenging benchmarks encompassing \textit{near-isometric} (FAUST~\cite{bogo_faust_2014}, SCAPE~\cite{anguelov_scape_2005}, SHREC~\cite{melzi_shrec_2019}, DeformingThings4D intra-~\cite{li_4dcomplete_2021}), \textit{non-isometric} (SMAL~\cite{zuffi_3d_2017}, DeformingThings4D inter-~\cite{li_4dcomplete_2021}), and \textit{topologically noisy} (TOPKIDS~\cite{lahner_shrec16_2016}) settings.
We use the more challenging re-meshed versions as established~\cite{donati_deep_2022,cao_unsupervised_2023}.

\subsection{Hybrid Basis in Different Frameworks}

To understand the efficacy of the proposed hybrid basis in various methodological settings, we use it instead of the LBO basis in three different methods spanning supervised (GeomFMaps~\cite{donati_deep_2020}), unsupervised (ULRSSM~\cite{cao_unsupervised_2023}), and axiomatic settings (Smooth Shells~\cite{eisenberger_smooth_2020}).
Due to inherent variability, we reproduce each experiment 5 times in both the baseline (LBO) and hybrid configuration, reporting the best results consistent with standard practices.
The total number of basis elements $k$ is kept fixed per method for all experiments; we replace only the highest-frequency LBO eigenfunctions with the elastic basis functions corresponding to the smallest eigenvalues.
Quantitative experimental results (c.f. \cref{tab:main_results}) are organized into sections (supervised, unsupervised, axiomatic), where we compare to competitive methods in the same category. 
Qualitative results are shown in \cref{fig:qualitative_topkids}, \cref{fig:qualitative_result_vertex} and in the supplementary.

\input{content/qualitative}

\input{content/vertex_map_result}

\textbf{GeomFMaps} \cite{donati_deep_2020} originally proposed the addition of a Laplacian regularization term to the FMNet framework, which has proven effective at enforcing isometric characteristics of the map calculated from \cref{eq:linear_lsq}.
We replace the LBO basis functions with the hybrid formulation, solving them separately as proposed in \cref{sec:method}.
For the elastic part of the functional map, we replace both the $E_\text{data}$ and $E_\text{reg}$ terms in the map optimization problem with our weighted variations.
We also regularize the ground truth supervision loss $\mathcal{L}_{\text{gt}}$ = $(C - C_\text{gt})$ with the weighted HS-norm.
The hybrid functional map is refined during inference to obtain dense point-to-point correspondences by performing a nearest-neighbor search in the hybrid vector space.
Following the recommendations of the original authors \cite{donati_deep_2020}, all final results in \cref{tab:main_results} are run at a spectral resolution of $k_{\text{LB}} = 20$, $k_{\text{Elas}} = 10$. %

\noindent\textbf{Results.}
We compare our results with those of GeomFMaps under the LBO basis and the supervised method FMNet~\cite{litany_deep_2017}.
Notably, the proposed hybrid basis outperforms LBO GeomFMaps in most settings, spanning near-isometric and non-isometric shape matching, where a particular benefit can be seen for SHREC'19 and SMAL, with a $2.3$ and $0.8$ improvement in mean geodesic error, respectively.

\textbf{ULRSSM} \cite{cao_unsupervised_2023}
has recently achieved SoTA performance in various challenging shape-matching settings.
We evaluate our proposed hybrid basis when used in ULRSSM instead of the pure LBO functional basis.
ULRSSM uses the functional map computation term described in~\cref{eq:linear_lsq}. 
Hence, we proceed to split the optimization problem as described in~\cref{sec:hybrid_basis} and adapt the elastic part with the proposed weighted formulation.
The authors of ULRSSM additionally regularize the functional map $C$ to preserve bijectivity $\mathcal{L}_{\text{bij}}$, orthogonality $\mathcal{L}_{\text{orth}}$, and a loss coupling functional and point-to-point maps $\mathcal{L}_{\text{couple}}$ in a differentiable manner.
For the elastic optimization, these are all reformulated in the HS-norm.
We use the same overall spectral resolution $k=200$ as the original implementation\cite{cao_unsupervised_2023}, with $k_{\text{LB}} = 140$, $k_{\text{Elas}} = 60$.
This choice of basis ratio is discussed in the appendix.

\noindent\textbf{Results.}
Using the hybrid basis instead of the LBO basis in ULRSSM results in notable performance improvements, even in near-isometric matching settings such as FAUST and SCAPE ~\cref{tab:main_results}.
Improvements are most significant in the non-isometric settings, including SMAL and inter-class DT4D-H, where the hybrid basis outperforms LBO with a geodesic error of $0.6$.
The most notable performance increase can be observed for TOPKIDS, where the hybrid basis yields a $45\%$ improvement in geodesic error.
Percentage-correct-keypoints (PCK) plots underscore these results (see ~\cref{fig:pck_plots}).

\textbf{Smooth Shells} \cite{eisenberger_smooth_2020} remains one of the most strongly performing axiomatic methods for spectral shape matching.
The method generates initial hypotheses for aligning a shape pair through a Markov-Chain Monte-Carlo (MCMC) step in a low-dimensional spectral basis ($k = 20$).
The algorithm then proceeds with an alternating optimization using both extrinsic and intrinsic information.
Following the principle that Laplacian eigenfunctions capture coarse shape features well, we perform the MCMC initialization in the LBO basis. 
During the hierarchical matching step, we extend the product manifold with an additional dimension consisting of the elastic basis.
We use a spectral resolution of $k_{\text{LB}} = 300$, $k_{\text{Elas}} = 200$, while the original implementation uses $k = 500$.

\noindent\textbf{Results.}
We observe that the performance with the proposed hybrid basis also leads to improved performance of Smooth Shells' over pure LBO, particularly for non-isometric and topologically noisy settings.
Notable improvements can be seen for the TOPKIDS and SMAL datasets, with a $3.3$ and $1.6$ improvement in mean geodesic error, respectively.

\input{content/pck_plots}

%% file: content/qualitative.tex
\begin{figure*}[t]
  \centering
  \vspace{-0.3cm}
  \includegraphics[width=\linewidth]{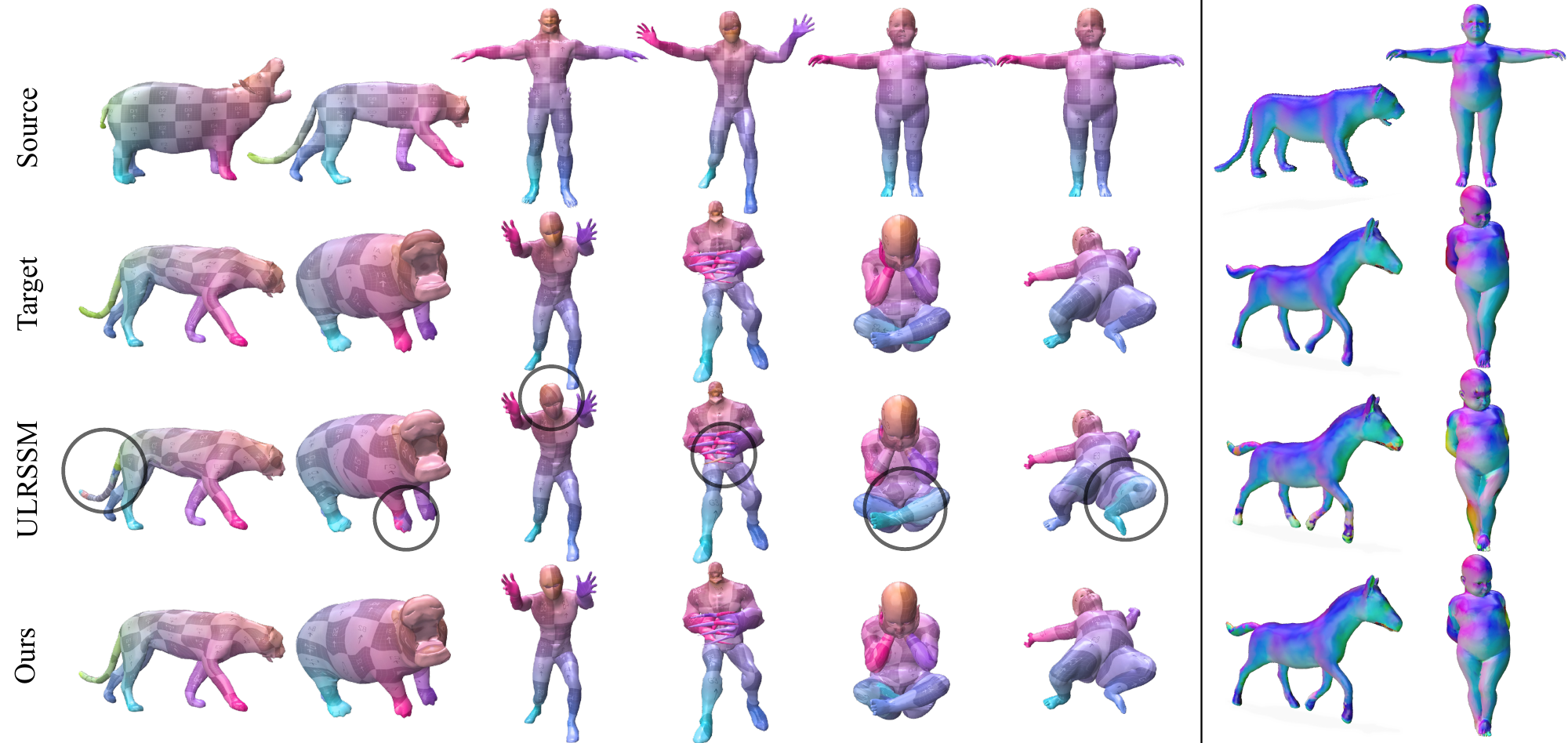}
  \vspace{-0.5cm}
  \caption{\textbf{Qualitative Results on SMAL, DT4D-H, and TOPKIDS.} Comparison of ULRSSM in the LBO basis and in the proposed hybrid basis. 
  Hybrid functional maps yield higher-quality correspondences, particularly under topological noise. 
  ULRSSM in the LB basis frequently creates coarse mismatches such as incorrectly assigning appendages, whereas the elastic basis better represents these details. The first six columns show texture transfer. The last columns transfer normals making the less accurate alignment of creases in ULRSSM visible.  }
  \label{fig:qualitative_topkids}
  \vspace{-0.2cm}
\end{figure*}

%% file: content/vertex_map_result.tex
\begin{figure}[t]
  \centering
  \vspace{-0.1cm}
  \includegraphics[width=0.5\textwidth]{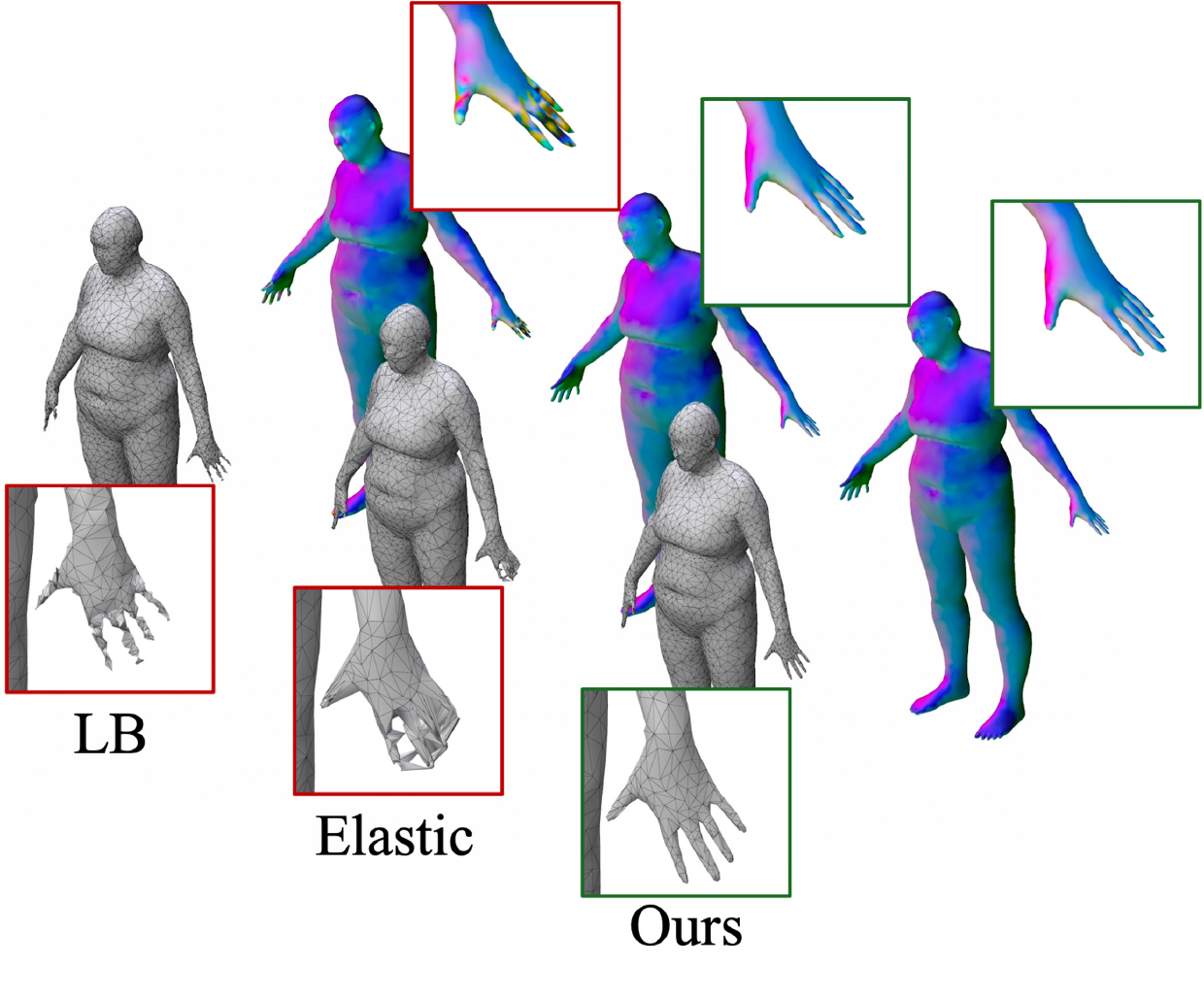}
  \vspace{-0.8cm}
  \caption{\textbf{Qualitative Results on FAUST}. Comparison of ULRSSM in the LBO, (orthogonalized) elastic, and proposed hybrid basis. Both normal (colored) and vertex transfer are visualized. The proposed hybrid basis yields accurate mappings, particularly on fine details such as the hands.
  }
  \vspace{-0.5cm}
  \label{fig:qualitative_result_vertex}
\end{figure}

%% file: content/pck_plots.tex
\begin{figure*}[t!]
    \centering
    \vspace{-0.2cm}
    \includegraphics[height=4.0cm]{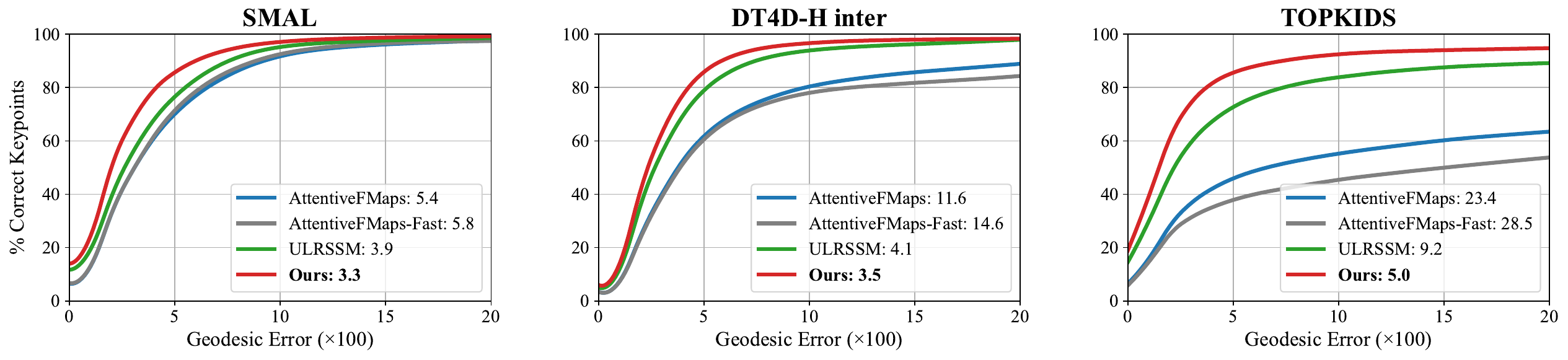}
    \vspace{-0.7cm}

    \caption{\textbf{Percentage-Correct-Keypoint Plots} depicting the geodesic error for state-of-the-art unsupervised methods on the datasets SMAL, DT4D-H inter, and TOPKIDS. We compare AttentiveFMaps, ULRSSM, and Hybrid ULRSSM(Ours).}
    \label{fig:pck_plots}
    \vspace{-0.2cm}
\end{figure*}

%% file: 05_ablations.tex
\subsection{Ablations and Implementation}
\label{sec:ablation}
We conduct two ablations to support the design choices regarding the generalization to the Hilbert-Schmidt (HS) norm in \cref{sec:method}, and the hybrid formulation.
All experiments are conducted on ULRSSM with $k = 200$ basis functions. 

To motivate the generalization to the HS norm, we consider two alternatives. 
The first involves using the standard functional map solver without making any adjustments to the non-orthogonal elastic basis.
The second alternative, we orthogonalize the basis using Gram-Schmidt under the inner product induced by $M$, making it directly usable in a standard FMap framework. 
Next, we compare the proposed hybrid formulation against using either pure LB or pure elastic basis functions.
Due to the complexity of the $k^2 \times k^2$ expansion under the HS-norm in~\cref{sec:method}, ULRSSM with 200 elastic basis becomes computationally intractable.
We therefore orthogonalize the elastic basis functions as an approximation to the proper adaptation of a fully elastic functional map.

\input{content/ablation.tex}

\noindent\textbf{Results.}
The results can be seen in~\cref{table:ablation} and~\cref{fig:motivation_pck}. 
We observe that using the standard Frobenius norm or orthogonalizing the elastic basis yields inferior results compared to the HS-norm adaptation. 
Furthermore, both pure LB and pure elastic basis adaptations perform worse than the proposed hybrid framework (c.f.~\cref{fig:motivation_pck}).
Interestingly, we observe that while the elastic basis functions achieve a superior detail alignment, and the LBO a better coarse alignment, the hybrid basis surpasses the performance in both regimes.
We conclude that both basis hybridization and the HS norm adaptation contribute to notable performance improvements.

\noindent\textbf{Implementation Details.}
Experiments are carried out in Pytorch 2.1.0 with CUDA version 12.1, except for Smooth Shells, which is run in Matlab based on the implementation provided by the authors.
Supervised and unsupervised methods are trained and evaluated on an NVIDIA A40.
A complete list of hyperparameters for each of the methods used is provided in the appendix.

%% file: content/ablation.tex
\begin{table}[b]
\vspace{-0.2cm}
\centering
\begin{tabular}{ccccc}
\toprule
LB & Elastic &  Adaptation & Geo. error ($\times$100)\\
\midrule
\xmark & \cmark & \textcolor{darkgreen}{\ding{70}} & intractable \\
\xmark & \cmark & - & $40.2 \pm 0.80$ \\
\xmark & \cmark & \textcolor{gray}{\large\textbf{+}} & $5.75 \pm 1.20$ \\
\rowcolor{orange!10} \cmark & \xmark & - & $5.15 \pm 0.99$ \\
\cmark & \cmark & - & $4.37 \pm 1.57$ \\
\cmark & \cmark & \textcolor{gray}{\large\textbf{+}} & $4.33 \pm 0.56$ \\
\rowcolor{green!10} \cmark & \cmark & \textcolor{darkgreen}{\ding{70}} & \textbf{3.83 $\pm$ 0.74} \\
\bottomrule
\end{tabular}
\vspace{-0.2cm}
\caption{
Ablation study of the proposed hybrid basis and the effect of two adaptations to the non-orthogonal elastic basis: generalization to the HS norm ({\textcolor{darkgreen}{\ding{70}}}) as proposed in \cref{sec:method} and orthogonalization ({\textcolor{gray}{\large\textbf{+}}}). 
Experiments are conducted with ULRSSM ~\cite{cao_unsupervised_2023} on the SMAL dataset at spectral resolution $k=200$. The green row represents \colorbox{green!10}{our approach}, and the orange row the \colorbox{orange!10}{original ULRSSM} ~\cite{cao_unsupervised_2023}. Experiments are conducted 5 times; mean $\pm$ stdev. is reported.}
\label{table:ablation}
\end{table}

%% file: 10_conclusion.tex
\section{Limitations and Conclusion}
\label{sec:conclusion}

This work explores the efficacy of combining basis functions originating from different operators for deformable shape correspondence. 
Our findings highlight the importance of accurately treating non-orthogonal basis functions to reflect the anisotropic metric on each shape. 
Imposing orthogonality on the basis functions shows improvement over naive adaptation but does not supplant proper mathematical adaptation of the optimization objectives. 
Additionally, the elastic basis functions underperform when used independently in a learned context; integrating it with low-frequency LBO basis functions significantly enhances spectral matching accuracy.

Solving the expanded $k^2 \times k^2$ system from~\cref{sec:method} leads to computational overhead; however, this is tractable for the elastic basis size of $60$. 
Performance gains in the expanded form justify this trade-off. 
Future research could potentially address partial shapes or noisy point clouds with non-orthogonal basis functions as these are active areas of interest~\cite{attaiki_dpfm_2021, cao_self-supervised_2023, huang_multiway_2022}.

Overall, the proposed hybrid functional mapping approach, leveraging both elastic and LBO eigenfunctions, exhibits notable performance in diverse settings, including isometric and non-isometric deformations and under topological noise. 
Our findings open new avenues for integrating various non-orthogonal basis functions into deep functional mapping frameworks, paving the way for further advances in spectral shape matching for challenging settings.

\paragraph*{Acknowledgements.} 
We thank the reviewers, Lei Li, and Dongliang Cao for their support and valuable feedback.
Zorah Lähner is funded by a KI-Starter grant from the Ministry of Culture and Science of the State of North Rhine-Westphalia.

%% file: 12_appendix.tex
\maketitlesupplementary

In the supplementary materials we first provide additional mathematical background in~\cref{sec:suppl_background}.~\cref{sec:suppl_HS} contains complete derivations for the generalization to the HS norm (Lemma~\ref{lemma:edata_induced_norm},~\cref{thm:ereg_hs_norm}). Detailed analysis for the hybrid formulation is included in~\cref{sec:suppl_hybrid_analysis}.  
We then detail datasets and splits used for evaluation, followed by further experimental details in~\cref{sec:implementation_details}.
Additional experiment results is provided at \cref{sec:suppl_additional_experiments}. 
Ablation studies concerning our design choices are provided in~\cref{sec:suppl_ablation}.
Finally, we present runtime analysis in~\cref{sec:runtime}, additional visualization of the basis embedding in~\cref{sec:suppl_basis_embedding} and additional qualitative results in~\cref{sec:suppl_qualitative}.

\input{content/appendix/math_background_suppl}

\input{content/appendix/derivations_suppl}

\section{Implementation Details}
\label{sec:implementation_details}

\begin{figure*}[t!]
  \centering
  \vspace{-0.9cm}
  \includegraphics[width=0.9\textwidth]{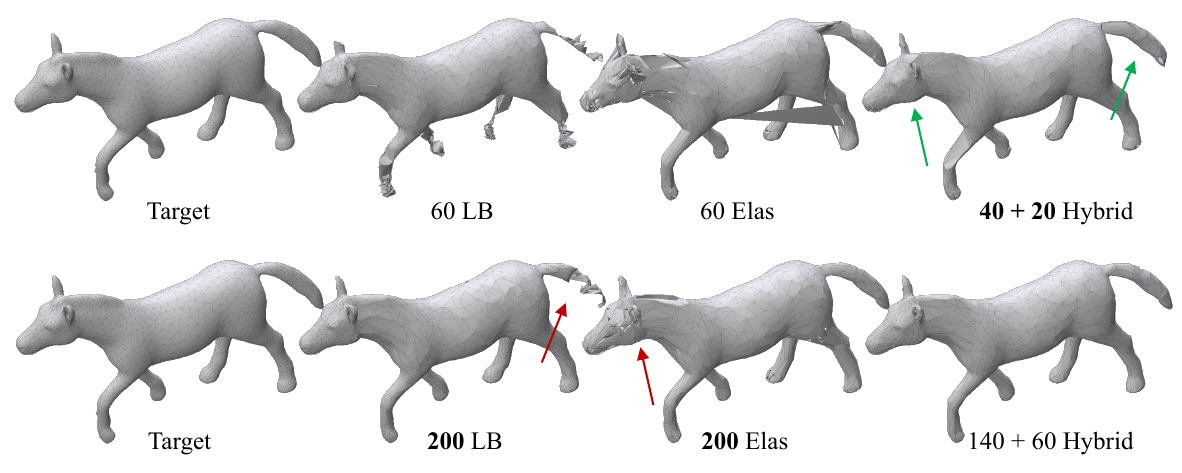}
  \vspace{-0.5cm}
  \caption{Comparison of correspondence visualizations by transferring vertex positions from the
source (horse) to the target shape (hippo). We examine the results from LB, Elastic,
and Hybrid basis functions by encoding and recovering \textbf{ground truth} point-to-point correspondences through functional map representations at
a spectral resolution of $k = 60$ and $k = 200$. Notably, the hybrid functional map representation can encode reasonable correspondences at a resolution of 60, while the other two methods struggle even at a resolution of 200.}
  \label{fig:sup_vertex_map}
\end{figure*}
\subsection{Datasets}
\label{sec:suppl_datasets}

We evaluate our method across near-isometric, non-isometric, and topologically noisy settings. 
Splits are chosen based on standard practices in the recent literature~\cite{donati_deep_2022,cao_unsupervised_2023}.

\paragraph*{Near-isometric:} The FAUST, SCAPE, and SHREC'19 datasets represent near-isometric deformations of humans, with 100, 71, and 44 subjects, respectively.
We follow the standard train/test splits for FAUST and SCAPE:
: 80/20 for FAUST and 51/20 for SCAPE.
Evaluation of our method on SHREC'19 is conducted with a model trained on a combination of FAUST and SCAPE inline with recent methods~\cite{li_learning_2022,donati_deep_2022,cao_unsupervised_2023}.
We use the more challenging re-meshed versions as in recent works.

\paragraph*{Non-isometric:}
The SMAL dataset features non-isometric deformations between 49 four-legged animal shapes from eight classes. 
The dataset is split 5/3 by animal category as in Donati et al.~\cite{donati_deep_2022}, resulting in a train/test split of 29/20 shapes.
We further evaluate the large animation dataset DeformingThings4D (DT4D-H)~\cite{li_4dcomplete_2021}, using the same inter- and intra-category splits as Donati et al.~\cite{donati_deep_2022}.
\paragraph*{Topological Noise:}
The TOPKIDS dataset~\cite{lahner_shrec16_2016} consists of shapes of children featuring significant topological variations and poses a significant challenge for unsupervised functional map-based works. 
Considering its limited size of 26 shapes, we restrict our comparisons to axiomatic and unsupervised methods and use shape 0 as a reference for matching with the other 25 shapes, following recent methods~\cite{eisenberger_smooth_2020, donati_deep_2022, cao_unsupervised_2023}.

\subsection{Experimental Details}
\label{sec:suppl_implementation}
In this section we provide additional details regarding the evaluation of our proposed hybrid basis from~\cref{sec:experimental}, including the axiomatic, supervised, and unsupervised settings.
Unless otherwise mentioned, implementations and parameters are left unaltered for the hybrid adaptation.

We first provide general details regarding learning and then the individual adaptations for each method.
Learned methods (GeomFMaps~\cite{donati_deep_2020} and ULRSSM~\cite{cao_unsupervised_2023}) are trained with PyTorch, using DiffusionNet as the feature extractor and WKS descriptors as input features, except for the SMAL dataset where we use XYZ signal with augmented random rotation as in recent methods~\cite{li_learning_2022,cao_unsupervised_2023}. 
The dimension of the output features is fixed at 256 for all experiments. 

For unsupervised learned methods, we propose the following linear annealing scheme for learning in a hybrid basis, as mentioned in~\cref{sec:learning_hybrid}.
\begin{align*}
    \mathcal{L}_{\text{total}} = \alpha\mathcal{L}_{\text{LB}} + &\mu \ \beta \ \mathcal{L}_{\text{Elas}} \\
    \alpha = \frac{1}{2} \cdot \frac{k^2}{(k_{\text{LB}})^2} \qquad &
    \beta = \frac{1}{2} \cdot \frac{k^2}{(k_{\text{Elas}})^2} 
\end{align*}

Where $k$ is the total spectral resolution.
The parameters $\alpha$ and $\beta$ ensure the losses are normalized w.r.t. the number of entries in the functional map similar to the approach of Li et al.~\cite{li_learning_2022}.
We increase $\mu$ over the first 2000 iterations so that the less-robust elastic basis functions do not adversely affect feature initialization.\\

In the following, \( C \) represents the block-functional map's elastic part for clarity.
Without loss of generality we let $C_{12}: \mathcal{F}(S_1) \to \mathcal{F}(S_2)$.
\paragraph*{Hybrid GeomFMaps.} We use 30 total eigenfunctions as in the original work.
For the hybrid adaptation, 20 LBO and 10 Elastic basis functions are used as the spectral resolution.
To compute the functional map, we use the standard regularized functional map solver and set \(\lambda = 1 \times 10^{-3}\) as in the original work~\cite{donati_deep_2020}.
For the hybrid adaptation, we empirically set \(\lambda = 5 \times 10^{-4}\) for the elastic solver.

\begin{figure*}[h]
  \centering
  \includegraphics[width=\textwidth]{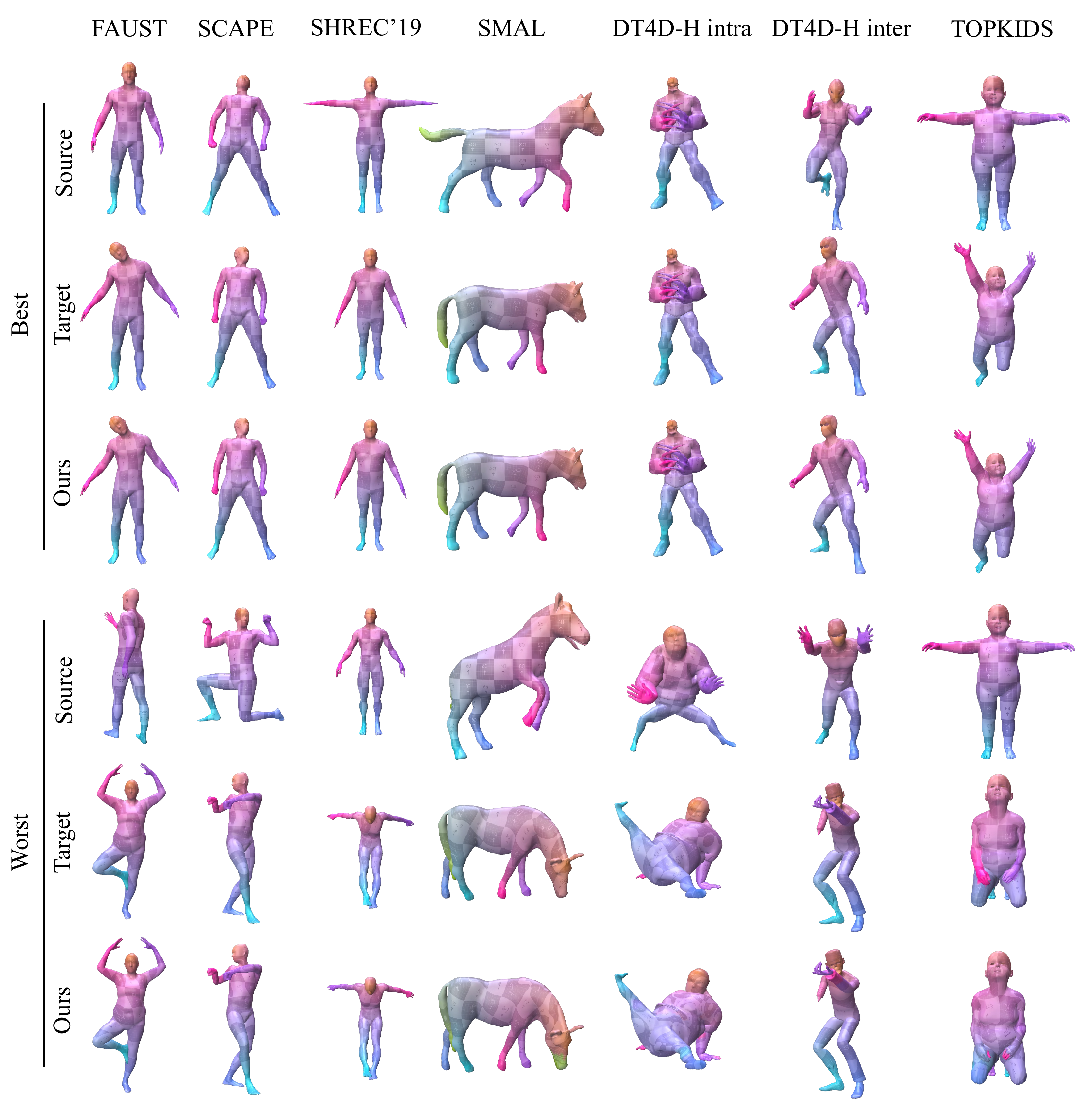}
  \caption{\textbf{Additional qualitative results} for the best and worst predictions of ULRSSM in the proposed hybrid basis (our method).}
  \label{fig:sup_quali_full}
\end{figure*}

GeomFMaps is supervised using a functional map constructed from the ground-truth correspondences. We thus adapt the elastic loss as follows (note here C refers to \( \mathcal{F}(S_1) \to \mathcal{F}(S_2) \) as the original work~\cite{donati_deep_2020}):

\begin{equation*}
\mathcal{L}_{\text{Elas}} = \left\| C - C_{\text{gt}} \right\|_{\text{HS}}^2 = \Vert \sqrt{M_{k,2}} (C - C_{\text{gt}}) \sqrt{M^{-1}_{k,1}} \Vert_F^2
\end{equation*}

\paragraph*{Hybrid ULRSSM.}
The ULRSSM baseline~\cite{cao_unsupervised_2023} uses a spectral resolution of $k=200$.
We keep the total spectral resolution fixed at $k=200$, using $140$ LBO and $60$ Elastic eigenfunctions.

For the functional map computation, we use the Resolvent regularized functional map solver~\cite{ren_structured_2019} for LB map block setting \(\lambda = 100\) as in the original work. 
Our adapted variant is weighted empirically with \(\lambda = 50\) for the elastic block. 
ULRSSM regularizes the functional map obtained from~\cref{eq:linear_lsq} with 3 losses: bijectivity, orthogonality, and a coupling loss with the point-to-point map.
The loss for the LBO functional map block $\mathcal{L}_{\text{LB}}$ is kept the same as the baseline method while we adapt the bijectivity, orthogonality, and coupling terms for the elastic block in the HS-norm.

While the bijectivity loss is left unchanged, we adapt the \( \mathcal{L}_{\text{orth}} \) term using the adjoint $C^*$ as follows, similar to Hartwig et al. in their adapted ZoomOut~\cite{hartwig_elastic_2023}:
\begin{align*}
\mathcal{L}_{\text{orth}} &= \left\|C_{12}^*C_{12} - I \right\|_{\text{HS}}^2 + \left\|C_{21}^*C_{21} - I \right\|_{\text{HS}}^2  \\
 &= \left\|C_{21}^*C_{21} - I \right\|_F^2 + \left\|C_{12}^*C_{12} - I \right\|_F^2
\end{align*}

We note that concerning the respective bijectivity and orthogonality losses, the operators $C_{12}C_{21} - I$ and $C_{21}^*C_{21} - I$ map to and from the same function space, thus the HS-norm is equivalent to the standard Frobenius norm and requires no non-uniform weighting.

The \( \mathcal{L}_{\text{couple}} \) term is given by:

\begin{align*}
\mathcal{L}_{\text{couple}} &= \left\|C_{12} - \Psi_{2}^{\dagger} \Pi_{21} \Psi_{1}\right\|_{\text{HS}}^2 + \left\|C_{21} - \Psi_{1}^{\dagger} \Pi_{12} \Psi_{2}\right\|_{\text{HS}}^2 \\
&= \left\|\sqrt{M_{k,2}}(C_{12} - \Psi_{2}^{\dagger} \Pi_{21} \Psi_{1})\sqrt{M^{-1}_{k,1}}\right\|_{F}^2 \\ 
&+ \left\|\sqrt{M_{k,1}}(C_{21} - \Psi_{1}^{\dagger} \Pi_{12} \Psi_{2})\sqrt{M^{-1}_{k,2}}\right\|_{F}^2
\end{align*}

For the definition of the point-to-point maps $\Pi_{21}$ and $\Pi_{12}$ we refer readers to the original method~\cite{cao_unsupervised_2023}.

Empirically, we set \( \lambda_{\text{bij}} = \lambda_{\text{orth}} = \lambda_{\text{couple}} = 1.0 \) for the LB part following the original work.
We keep these parameters the same for the elastic block except setting $\lambda_{\text{orth}} = 0.0$ as we observed the orthogonality constraint adversely affects the method's performance. 

\paragraph*{Hybrid SmoothShells.}
To demonstrate how the proposed hybrid basis can be used in an axiomatic method, we adapt the method SmoothShells~\cite{eisenberger_smooth_2020} with the minimally needed changes.

The initialization of SmoothShells consists of a low-frequency MCMC alignment. 
We keep this step as-is and do not replace the LBO smoothing with the hybrid eigenfunctions because the elastic eigenfunctions cannot achieve low-frequency smoothing by design. 
We fix the random seed and re-run the baseline Smooth Shells and the hybrid version with the same MCMC initialization to rule out noise.

The main idea of Smooth Shells~\cite{eisenberger_smooth_2020} is to achieve a coarse-to-fine alignment by iteratively adding higher-frequency LBO eigenfunctions to the intrinsic-extrinsic embedding. 
Instead of only adding LBO eigenfunctions in a new iteration, we add a ratio of LBO and elastic eigenfunctions.  %
As in the other adaptations, we keep the total number of basis functions $k=500$ fixed.
We then empirically replace the highest $200$ LBO eigenfunctions with elastic eigenfunctions, modifying the product embedding to be:

\begin{align*}
    \mathbf{X}_k &:= \left( \Phi_{1,k}, \Psi_{1,k}, X_k, \mathbf{n}_k^{S_1} \right) \in \mathbb{R}^{n_1 \times (k+6)} \\
    \mathbf{Y}_k &:= \left( \Phi_{2,k}, \Psi_{2,k}, Y_k, \mathbf{n}_k^{S_2} \right) \in \mathbb{R}^{n_2 \times (k+6)} \\
\end{align*}

\noindent where we use our notation of the LBO and elastic basis functions. $\mathbf{X}_k,\mathbf{Y}_k$ are the product embeddings for shape $S_1$ and $S_2$ with $X_k,Y_k$ the respective smoothed cartesian coordinates, and $\mathbf{n}_k$ the outer normals on each shape.
The rest of the optimization follows directly from \cite{eisenberger_smooth_2020}.

\section{Additional Experiments}
\label{sec:suppl_additional_experiments}
\subsection{Near-isometric Cross-dataset Generalisation}
We provide additional experimental results in near-isometric cross-dataset scenarios on FAUST, SCAPE, and SHREC'19 datasets using ULRSSM ~\cite{cao_unsupervised_2023} and our hybrid adaptation in ~\cref{tab:near-iso}.
Using the hybrid basis also improves performance under these settings.

\input{content/near_isometric}

\begin{table}[b!]
\centering
\footnotesize
\renewcommand{\arraystretch}{1.2}
\begin{tabular}{lcc} %
\Xhline{0.75pt}
\;\textbf{Geo.err. (×100)}\;\;\;\; & \textbf{FAUST} & \textbf{SMAL} \\ \hline
\;\;\;200\;\;\;\;\;\;\; LB & 2.1 & 7.6 \\ \hline
\;\;\;200\;\;\;\;\;\;\; Elastic & 2.0 & 5.9 \\ \hline
100+100\;\; Hybrid & \textbf{1.6} & \textbf{5.1} \\ \hline
\end{tabular}
\caption{Results of applying the proposed hybrid basis to ZoomOut~\cite{melzi_zoomout_2019,hartwig_elastic_2023}, compared to pure LB and pure elastic bases. 
We initialize each method with a prediction from Hybrid GeomFmaps (from~\cref{tab:main_results}) and carry out the spectral upsampling from $k=30$ to $k=200$, particularly, we upsample from $k=20+10$ to $100+100$ for Hybrid ZoomOut. 
We demonstrate results on both near-isometric (FAUST) and non-isometric (SMAL) datasets.}
\label{tab:zoomout}
\end{table}

\subsection{Hybrid ZoomOut Refinement}

We provide experimental results for ZoomOut (see ~\cref{tab:zoomout}) when refining a Hybrid GeomFMaps correspondence prediction to $k=200$ in both near-isometric (FAUST) and non-isometric (SMAL) settings. 
It is worth noting that the performance of the hybrid ZoomOut algorithm relies on good initialization, similar to the original method \cite{melzi_zoomout_2019,hartwig_elastic_2023}, and the performance can vary significantly given a poor initialization.

\section{Further Ablation Studies}
\label{sec:suppl_ablation}

In this section, we present the results of several ablation studies, focusing on key design choices: separating the optimization problems from~\cref{thm:seperable_block_matrix}, the ratio of the hybrid basis, and our training strategies.

\subsection{Hybridization and Separation.}
We argue that the proposed hybrid notion is conceptually important; it is essential for optimization and training stages, including point map conversion and training learned features. 
Preventing inter-basis matchings by setting the off-diagonal blocks of a hybrid functional map to 0 serves as a strong regularization, which is also computationally efficient (see ~\cref{fig:ablation_block}).
The proposed hybrid formulation generally raises questions on the perspective of \textit{hybridization} and which components of a hybrid FM architecture could be optimized separately. 
To further support our design decisions, we provide ablation experiments for this regularization effect and analyze why training two entirely separate networks is not optimal.

\paragraph*{Separating the Optimization}
\label{sec:block_suppl_ablation_exp}
As detailed in~\cref{thm:seperable_block_matrix}, setting the off-diagonal blocks of the hybrid functional map to $\mathbf{0}$ is equivalent to solving the optimization problems separately.
Here, we show this is both important computationally and in terms of regularization.

To demonstrate its regularization effect, we conduct an experiment in ULRSSM on the FAUST dataset comparing solving a hybrid functional map with our proposed method against naively via a single solve. 
This experiment is conducted with the orthogonalized elastic basis as a proof-of-concept, as solving a full-dimensional ($k=200$) hybrid functional map from the $k^2 \times k^2$ system would be prohibitively expensive.
\begin{figure}[h!]
\centering
\includegraphics[width=\linewidth]{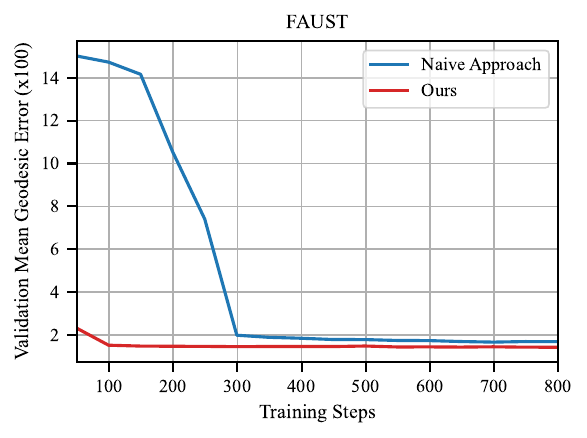}
\caption{Ablation study concerning separate optimization of the block matrix on FAUST with Hybrid ULRSSM. 
The y-axis depicts validation error, while the x-axis shows training steps.
Separating the optimization problems primarily leads to faster convergence.}
\label{fig:ablation_block}
\end{figure}\\
\noindent\textbf{Results}. The results of this ablation are depicted in~\cref{fig:ablation_block}. 
We observe that solving the two maps separately yields notably faster convergence compared to the naive approach with a marginal performance advantage.
This suggests that a block-diagonal functional map is desirable; restricting inter-basis matches leads to faster convergence.
Separately solving the optimization problems can be interpreted as a strong regularization of the off-diagonal blocks, reducing the search space.

\paragraph*{Separating the Networks.}

During the training of a deep hybrid FM, features are updated by gradients passed through both map solves.
\cref{tab:separate_network} includes an additional experiment where we instead train two entirely separated feature networks, only concatenating features during inference to obtain correspondences. The performance falls behind our approach, showing that we separate only where meaningful.

\begin{table}[h!]
\centering
\footnotesize
\renewcommand{\arraystretch}{1.2}
\begin{tabular}{ccc}
\Xhline{0.75pt}
\textbf{Geo.err.(×100)} & \textbf{Hybrid ULRSSM} & \textbf{Hybrid GeomFmaps} \\
\hline
Separate Network & 3.9 & 10.9 \\
Ours & \textbf{3.3} & \textbf{7.6} \\
\hline
\end{tabular}
\caption{Ablation study on whether separating neural networks for feature training is effective to our proposed framework. Instead of using one shared weight neural network as feature extracor, consider training two completely separate networks and concatenate the features only at test time. Note that under this scheme, our annealing training strategies in~\cref{sec:training_strategy} are no longer meaningful. Experiment conducted on SMAL dataset with our two adapted learning methods. Best result is reported. }
\label{tab:separate_network}
\end{table}

\subsection{Basis Ratio}

To further validate the effectiveness of our choice of hybridizing between the LB (Laplace-Beltrami) and Elastic eigenfunctions, we conduct extensive ablation experiments showcasing the performance of different ratios of hybridized basis.

In all these experiments, we again fix the total number of basis functions used as $k$ while replacing the  highest frequency LB basis with the Elastic eigenfunctions corresponding to the smallest eigenvalues.
This follows the intuition that the low-frequency LB basis functions enable coarse shape alignment while failing to capture fine details, while optimizing in the hybrid basis enables alignment to thin structures and high curvature details better than in the pure LBO basis.
We conduct two ablations to demonstrate such a choice; both experiments were carried out on the SMAL dataset for its challenging non-isometry and practical relevance as a stress test.

\begin{figure}[h!]
  \centering
  \includegraphics[width=0.5\textwidth]{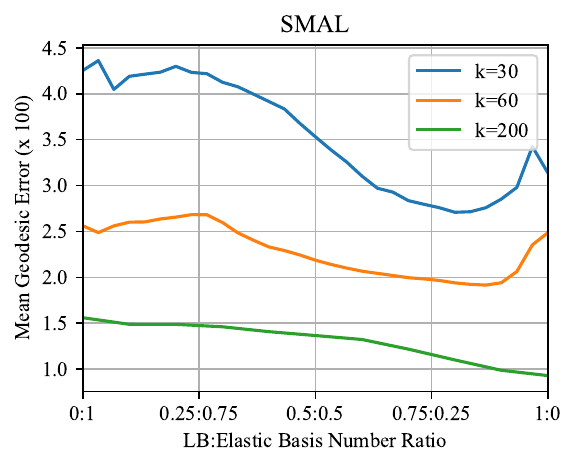}
  \caption{Point-to-point correspondence recovery in a hybrid basis, using a varying ratio of LBO and elastic eigenfunctions. We present several different total spectral resolutions k. Units are in geodesic distance (×100).}
  \label{fig:ablation_gt}
\end{figure}

\paragraph*{Point-to-Point Map Recovery}
For ground-truth map recovery, a functional map is obtained by projecting ground-truth point-to-point correspondences into the spectral domain.
Subsequently, the point-to-point correspondences are reconstructed using a nearest neighbor search, upon which the discrepancy with the ground-truth point-to-point map can be measured by geodesic error.
This simple experimental scenario enables a convenient way to measure the expressiveness of a functional map; %
In~\cref{fig:ablation_gt}, we consider the hybrid basis composed with a varying ratio of LBO and elastic basis functions and a different number of total basis functions: $k=30, 60$, and $200$. 
We measure the mean geodesic error between the ground-truth point-to-point map and the recovery from the hybrid functional maps.

\noindent\textbf{Results.} The hybridized basis can notably better represent the ground truth for $k=30$ and $60$. 
We observe an optimum of around 80\% LBO and 20\% elastic eigenfunctions. 
This phenomenon diminishes at $k=200$, suggesting the LB basis functions can indeed represent fine details with a sufficiently high number of basis functions.
However, ground-truth recovery does not necessarily represent the setting where features are learned through backpropagation of the functional map loss.
Our experiments indicate that learned pipelines cannot leverage the high-frequency LBO eigenfunctions to represent fine extrinsic details as effectively as the elastic basis functions, even with a large total number of basis functions.
We, therefore, conduct a similar ablation in the learned setting.

\begin{figure}[h!]
  \centering
  \includegraphics[width=0.5\textwidth]{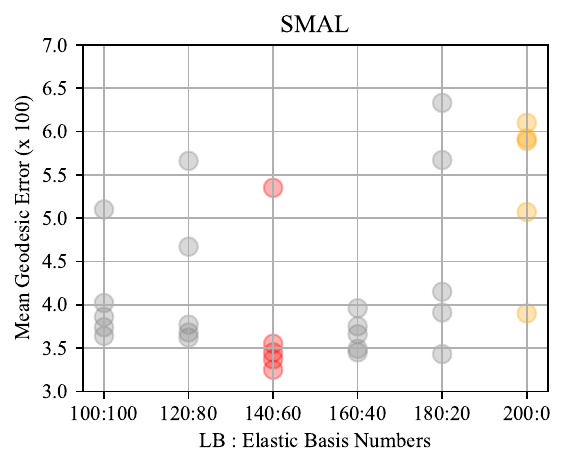}
  \caption{Hybrid ULRSSM is evaluated on SMAL using a varying ratio of LBO and elastic eigenfunctions. 
  Our basis ratio can be observed in red, with the baseline pure LBO implementation in yellow.}
  \label{fig:ablation_ratio}
\end{figure}

\paragraph*{Learned Setting.}
In~\cref{fig:ablation_ratio}, we consider the hybrid ULRSSM method with a fixing total basis number of $k = 200$ under different hybrid basis ratios with a step size of $20$. 
Due to the high-order polynomial computational increase (see ~\cref{sec:runtime}) for the $k^2 \times k^2$ system, we limit the number of elastic basis functions to less than $100$ (Point Map Recovery from GT also indicates inferior performance outside of this regime). 
We run each experiment 5 times to eliminate inherent noise and report all results.

\noindent\textbf{Results.} Here, we observe that a ratio of around 140:60 is optimal; 
the hybridized basis (red) shows consistent performance improvements over the baseline (orange) and other basis ratios.

\subsection{Training Strategies}
\label{sec:training_strategy}
Training a reliable shape correspondence estimation pipeline through hybrid functional maps involves several key modeling decisions.
Both the linearly increasing scheduler for the elastic loss during training and normalizing factors for both Laplace-Beltrami (LB) and elastic losses play a large role in the obtained performance increases.

As mentioned in~\cref{sec:method}, we observed the elastic basis functions are not robust to uninitialized features.
Easing in the elastic loss after feature initialization in the LBO basis mitigates convergence to undesirable local minima. 
Furthermore, the loss of each component in the hybrid functional map is normalized according to the number of matrix elements for this component, an important hyperparameter to balance the two blocks.
During the ablation studies presented in~\cref{fig:ablation_optimizatoin}, we selectively eliminate each one of these factors from our model and measure the mean geodesic error.
We further demonstrate that fine-tuning from a pre-trained LBO checkpoint is ineffective, likely converging to local minima.
\begin{figure*}[t!]
  \centering
  \vspace{-0.5cm}
  \begin{subfigure}[b]{0.45\textwidth}
    \centering
    \includegraphics[width=\textwidth]{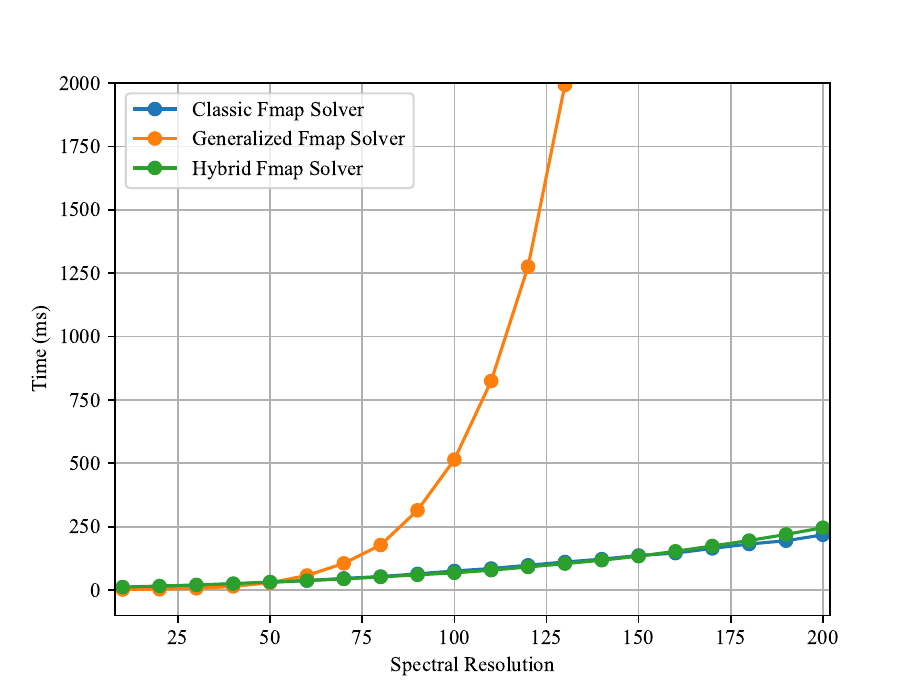}
    \label{fig:fmap_time}
  \end{subfigure}
  \hfill
  \begin{subfigure}[b]{0.45\textwidth}
    \centering
    \includegraphics[width=\textwidth]{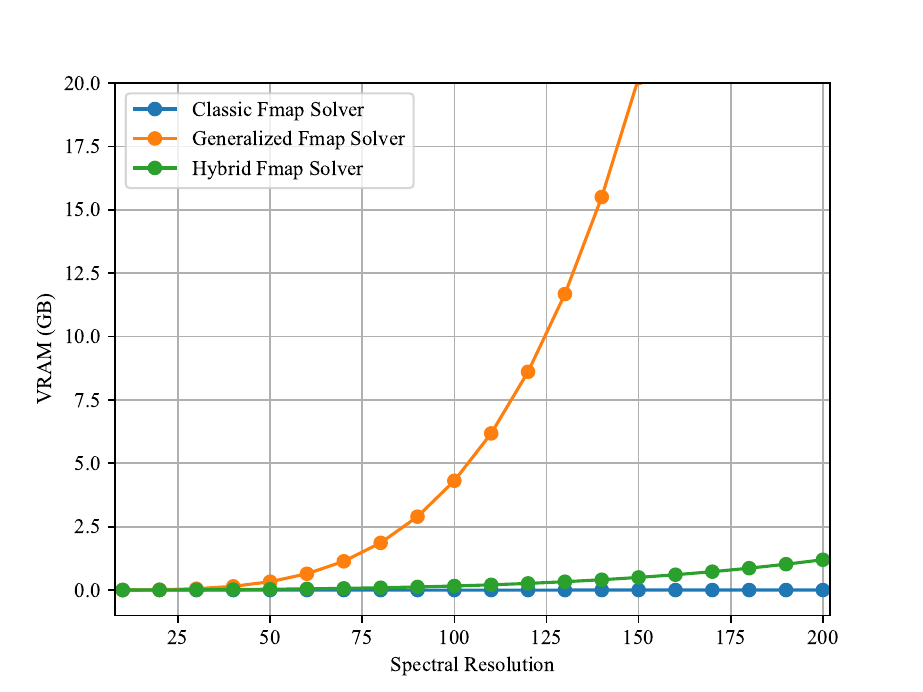}
    \label{fig:fmap_vram}
  \end{subfigure}
  \vspace{-0.5cm}
  \caption{The computation cost comparison of deep functional map solvers between classic, generalized, and hybrid solutions: runtime per solve (left), and VRAM consumption (right). Results were obtained in PyTorch on an NVIDIA A40.
}
  \label{fig:fmap_cost_combined}
\end{figure*}
\noindent\textbf{Results.} The results in ~\cref{fig:ablation_optimizatoin} show that each component is indeed important for our final model. 
Fine-tuning from a checkpoint or training without normalization yields inferior results. 
Furthermore, except for a single outlier, our approach converges to a significantly lower minimum than learning without the linear-annealing strategy.

\begin{figure}[h!]
  \centering
  \includegraphics[width=0.5\textwidth]{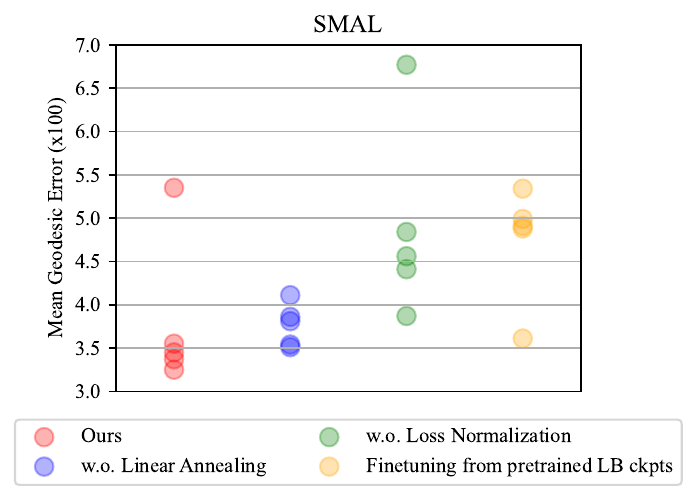}
  \caption{Ablation study of optimization strategies using Hybrid ULRSSM on SMAL. Five random runs are shown for each training setting.}
  \label{fig:ablation_optimizatoin}
\end{figure}

\begin{table}[htbp]
  \centering
  
  \renewcommand{\arraystretch}{1.5} %
  \begin{tabular}{c c} 
    \hline
    \textbf{Method} & \textbf{Runtime} \\ \hline 
    ULRSSM & 610.70 $\pm$ 45.20 ms\\ \hline 
    Hybrid ULRSSM & 623.19 $\pm$ 32.03 ms\\ \hline 
  \end{tabular}
    \caption{
    Comparison of per-iteration runtime on SMAL over 100 training iterations (mean $\pm$ st. dev). 
    }
    \label{tab:runtime}
\end{table}
\section{Runtime Analysis}
\label{sec:runtime}

We provide our runtime analysis for the Hybrid ULRSSM method in SMAL dataset under Table \ref{tab:runtime} and \cref{fig:fmap_cost_combined}.
Results are obtained on an NVIDIA A40.
Our hybrid adaptation incurs minimal runtime overhead while yielding significant performance gains despite needing to solve an expanded $k^2 \times k^2$ system.
This can be explained by analyzing the complexity.
Assuming the complexity of solving a linear system for an $k \times k$ matrix is $\mathcal{O}(k^3)$, solving the combined optimization problem costs $\mathcal{O}(k^4)$ flops (since we solve $k$ separate $k \times k$ systems.
Solving the separate optimization problem costs $\mathcal{O}((k-l)^4 + l^6)$ flops, which for the total spectral resolution $k = 200$ and elastic eigenfunctions $l = 60$ is only one order of magnitude larger.

\section{Additional Figures}
\label{sec:suppl_basis_embedding}

We provide additional visualization of the basis embeddings in~\cref{fig:sup_embedding}. The figure depicts 3D embeddings of the first three basis functions on the shape in different combinations.
The visualization suggests that while the elastic basis lacks global structure, it naturally extends the smooth approximation of the LBO basis to incorporate geometric details (creases) when hybridized. \\
Additionally, we provide further comparison of correspondence visualizations by transferring vertex positions from the source (horse) to the target shape (hippo), the results of LB, Elastic, and Hybrid basis functions were compared by encoding and recovering ground truth correspondences through functional map representations. As illustrated in \cref{fig:sup_vertex_map}, the hybrid functional map representation notably encodes reasonable correspondences already at a spectral resolution of 60, whereas the other two methods still face challenges even at a resolution of 200.
\section{Additional Qualitative Results}
\label{sec:suppl_qualitative}

We provide additional qualitative results in~\cref{fig:sup_quali_full} on each dataset. 
We evaluate and visualize the best and worst predictions of ULRSSM in the proposed hybrid basis.

%% file: content/appendix/math_background_suppl.tex
\section{Mathematical Background}
\label{sec:suppl_background}
\input{content/notation.tex}

\subsection{Elastic Energy}

We defer to Hartwig et al. ~\cite{hartwig_elastic_2023} for a complete definition of the previously described elastic energy~\cite{aubry_wave_2011,heeren_exploring_2014,hartwig_elastic_2023}.
For all our experiments, we use the same elastic energy hyperparameters as Hartwig et al.~\cite{hartwig_elastic_2023}, including a bending weight of $10^{-2}$.
We can then solve the generalized eigenvalue problem for the Hessian of the elastic energy at the identity to obtain the basis functions $\Psi$.
\begin{equation*}
\text{Hess} \; \mathcal{W}_S[\text{Id}] v_\lambda = \lambda M v_\lambda
\end{equation*}

\subsection{Problem Setting}
In the non-rigid correspondence literature, descriptors $D_i$ are commonly characterized as functions over the shapes $S_i$.
In the discretized setting, many operations reduce to matrix-vector products.
However, to derive the proper operations weighted by the non-uniform weight matrices $M$ in the regularization of the functional map, we utilize the more general Hilbert space setting.

We assume that all functions on the spaces $\mathcal{F}(S_i)$ are $L^2$ integrable: 
$$L^2(\Omega) := \left\{ f:\Omega \rightarrow \mathbb{R} \mid \int_{\Omega} |f(x)|^2 \, dx < \infty \right\}$$

\noindent Then the inner product on each space $\mathcal{F}(S_i)$ is given by $\langle \cdot, \cdot \rangle_M$ in the space induced by $M$:

\begin{equation*}
\label{eq:induced_inner_product}
\langle x, y \rangle_{M} = \int_{\Omega} x(t) y(t) \, dM(t) \stackrel{?}{=}
 x^\top M y
\end{equation*}

where in the last step we emphasize that the discretized operations reduce to matrix-vector multiplications in the finite-dimensional setting.

We use the definition of the Hilbert-Schmidt norm for a general operator $A$ between (unweighted) Hilbert spaces~\cite[Sec 3.4]{hartwig_elastic_2023}:

\begin{equation*}
\label{eq:hs_norm}
\Vert A \Vert_{HS} := \sqrt{\text{tr}(A^* A)}
\end{equation*}

When $A: \mathcal{F}(S_1) \to \mathcal{F}(S_2)$ (the spaces under the anisotropic metric), we have the following equivalence with the Frobenius norm~\cite{hartwig_elastic_2023}:

\begin{align}
\label{eq:hs_frob_equiv}
    \|A\|^2_{HS} :&= \text{tr}(M_{k,1}^{-1} A^\top M_{k,2} A ) \\ \nonumber
&= \text{tr}(\sqrt{M_{k,1}^{-1}} A^\top \sqrt{M_{k,2}}\sqrt{M_{k,2}} A \sqrt{M_{k,1}^{-1}}) \\ \nonumber
&= \left\| \sqrt{M_{k,2}} A \sqrt{M_{k,1}^{-1}} \right\|_F^2 \nonumber
\end{align}

%% file: content/notation.tex
\setlength{\tabcolsep}{5pt}  %
\begin{table}[b!]
\vspace{-1em}
\small\centering
    \label{table:notation}%
    \begin{tabularx}{\columnwidth}{lp{5.6cm}}
        \toprule
        \textbf{Symbol} &\textbf{Description} \\
        \toprule
        $\mathcal{S}_1, \mathcal{S}_2$ &3D shapes (triangle mesh) with $n_1,_2$ verts \\
        $M_i$ & mass matrix on shape $i$\\
        $D_i$ & vertex-wise descriptors for shape $i$ \\
        $\Delta_i$ &Laplacian operator applied to shape $\mathcal{S}_i$ \\
        $\mathcal{W}_S[\cdot]$ &Elastic energy associated with $\mathcal{S}_i$ \\
        $\Phi_i$ & eigenbasis of Laplacian matrix $\Delta_{i}$ \\
        $\Psi_i$ & eigenbasis of Elastic Hessian $\text{Hess}\mathcal{W}_S[I]$ \\
        $C_{ij}$ &functional map between shapes $\mathcal{S}_i$ and $\mathcal{S}_j$\\
        $P{_{ij}}$ & point-wise map between shapes $\mathcal{S}_i$ and $\mathcal{S}_j$\\
        $k$ & the total spectral resolution\\
        $||\cdot||_{\{2,F,HS\}}$ & the L2, Frobenius, and HS norms\\
        \bottomrule
    \end{tabularx}
    \caption{Summary of notations used in this work.}

\end{table}

%% file: content/appendix/derivations_suppl.tex
\section{Generalization to the Hilbert-Schmidt Norm}
\label{sec:suppl_HS}
\subsection{Derivation of~\cref{eq:linear_lsq} for General Hilbert Spaces.}
\label{sec:suppl_derivations}

\begin{proof}[Proof of~\cref{lemma:edata_induced_norm}]

The data term can be interpreted as the difference of the descriptor functions $D_1, D_2 \in \mathcal{F}(S_2)$ after $D_1$ was transferred to $\mathcal{F}(S_2)$ via the functional map $C$. 
We denote the first k eigenfunctions by $\Psi_{k,i}$ and the coefficients of $D_i$ projected into the basis spanned by these eigenfunctions by $D_{\Psi_i}:= \Psi_{k,i}^\dagger D_i$. 
We then have the following:

\begin{align*}
    &\Vert CD_{\Psi_1} - D_{\Psi_2} \Vert_{M_{k,2}} \\
    =\ &\sqrt{\langle CD_{\Psi_1} - D_{\Psi_2}, CD_{\Psi_1} - D_{\Psi_2} \rangle_{M_{k,2}}} \\
    = &\sqrt{\text{tr}(( CD_{\Psi_1} - D_{\Psi_2} )^\top M_{k,2} (CD_{\Psi_1} - D_{\Psi_2} ))}\\
\end{align*}

\noindent where we use the definition of the inner product~\cref{eq:induced_inner_product}, the cyclicity of the trace.
The identity then follows by splitting $M_{k,2} = \sqrt{M_{k,2}} \sqrt{M_{k,2}}$ and applying the definition of the Frobenius norm again, and using that $M_{k,2}$ is symmetric.

\end{proof}

As previously established~\cite{donati_deep_2020}, the energy in~\cref{eq:linear_lsq} can be solved for $C$ in closed form by solving $k$ different $k \times k$ linear systems (for each row of $C$).
In our case, the mass matrices $M$ prohibit this, requiring an expansion to a $k^2 \times k^2$ system.
This expansion is detailed below.

\begin{proof}[Proof of ~\cref{thm:ereg_hs_norm}]

Let \( S_1 \) and \( S_2 \) be Hilbert spaces defined on two shapes associated with mass matrices \( M_{k,1} \) and \( M_{k,2} \), respectively, which induce the inner product on each space. 
Let \( \Lambda_1 \) and \( \Lambda_2 \) be the diagonal matrices of eigenvalues of the respective linear operator on \( S_1 \) and \( S_2 \), and let \( C \): $\mathcal{F}(S_1) \to \mathcal{F}(S_2)$ be a linear map between the function spaces. 
The weighted linear operator commutativity regularization term can be expressed using the Hilbert-Schmidt norm as follows:

\begin{flalign*}
||(C \Lambda_1 &- \Lambda_2 C)||_{HS}^2 \\
& = \text{tr}(M_{k,1}^{-1}(C \Lambda_1 - \Lambda_2 C)^\top M_{k,2} (C \Lambda_1 - \Lambda_2 C)) \\
&= \left\| \sqrt{M_{k,2}} \left( C \Lambda_1 - \Lambda_2 C \right) \sqrt{M_{k,1}^{-1}} \right\|_{\text{F}}^2 \\
&= \left\| \sqrt{M_{k,2}} C \Lambda_1 \sqrt{M_{k,1}^{-1}} - \sqrt{M_{k,2}} \Lambda_2 C \sqrt{M_{k,1}^{-1}} \right\|_{\text{F}}^2 \\
\end{flalign*}

\noindent where we apply the definition of the HS-norm~\cref{eq:hs_frob_equiv}, the definition of the Frobenius norm, and multiply out the terms.

Now, we can use the definition of the Kronecker product for matrices $E$, $F$, $G$:
\begin{equation*}
\label{eq:kron_vec}
    \text{vec}(EFG) = \left( G^\top \otimes E \right) \text{vec}(F)
\end{equation*}

\noindent to expand and rearrange this into the form $||\zeta x||_F$ for a matrix $\zeta$ and vector $x := \text{vec}(C)$:

\begin{flalign*}
\lVert \sqrt{M_{k,2}} &\left( C \Lambda_1 - \Lambda_2 C \right) \sqrt{M_{k,1}^{-1}} \rVert_{\text{F}}^2 \\
= \ &\lVert ((\Lambda_1 \sqrt{M_{k,1}^{-1}}) \otimes \sqrt{M_{k,2}} \ - \\
&\sqrt{M_{k,1}^{-1}} \otimes (\sqrt{M_{k,2}} \Lambda_2)) \text{vec}(C) \rVert_F^2
\end{flalign*}

\end{proof}

\subsection{Solving the Combined Optimization Problem}

To solve $E(C)$ for a vectorized functional map $C$, $E_\text{data}$ must be expanded similarly. Using \cref{eq:kron_vec} we have:

\begin{flalign*}
& {\Vert \sqrt{M_{k,2}} (CD_{\Psi_1} - D_{\Psi_2}) \Vert}_F \\
& = {\Vert vec(\sqrt{M_{k,2}} CD_{\Psi_1}) - vec(\sqrt{M_{k,2}} D_{\Psi_2}))\Vert}_2 \\
& = {\Vert ((\sqrt{M_{k,2}} D_{\Psi_1} )^\top \otimes I) vec(C) - vec(\sqrt{M_{k,2}}D_{\Psi_2}) \rVert}_2
\end{flalign*}

\noindent Where we use the fact that the Frobenius norm of a matrix is just the $L_2$ norm of its stacked column vectors and the definition of the Kroeneker product. 
Combining the expanded forms of $E_\text{data}$ and $E_\text{reg}$ and observing the first variation of $E(C)$ yields a $k^2 \times k^2$ linear system which can be solved for $C$:

\begin{flalign*}
& (A^\top A + \lambda \zeta^\top \zeta) vec(C) - A^\top vec(B) = 0
\end{flalign*}
Here, we made the following substitutions for readability:
\begin{flalign*}
& A = ( \sqrt{M_{k,2}} D_{\Psi_1})^\top \otimes I \\
& B = \sqrt{M_{k,2}}D_{\Psi_2} \\
& \zeta = (\Lambda_1 \sqrt{M_{k,1}^{-1}}) \otimes \sqrt{M_{k,2}} - \sqrt{M_{k,1}^{-1}} \otimes (\sqrt{M_{k,2}} \Lambda_2)
\end{flalign*}

\section{Hybrid Formulation}
\label{sec:suppl_hybrid_analysis}
\subsection{Basis Non-orthogonality}
\begin{figure}[h!]
  \centering
  \vspace{-0.5cm}
  \includegraphics[width=0.5\textwidth]{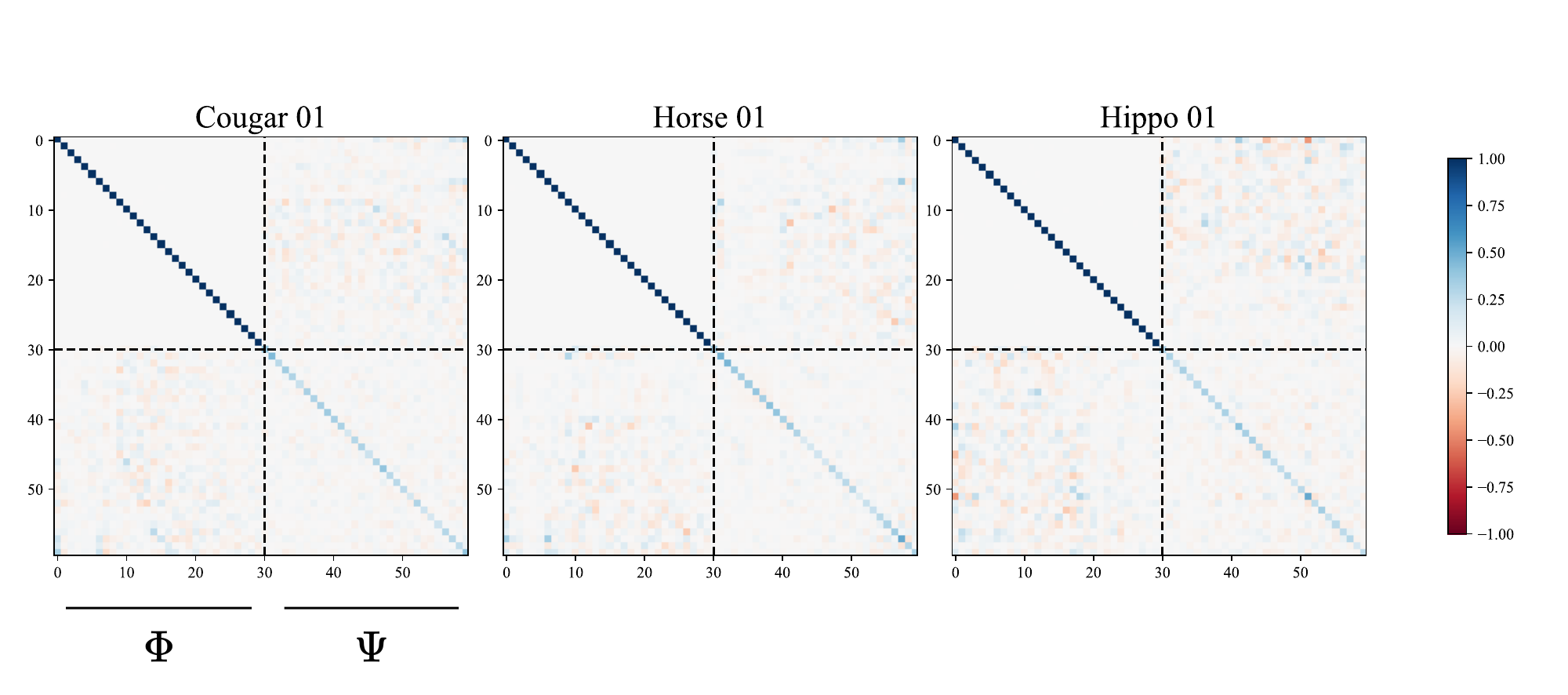}
  \caption{Matrix of inner product between hybrid basis $M_{k_{\text{total}}} = [\Phi \Psi]^\top M [\Phi \Psi]$, with $k_{\text{total}}=30+30$. We show the heatmap of the resulting matrix from three animal shapes of the SMAL dataset. }
  \label{fig:inner_product_matrix}
\end{figure}
To better understand the non-orthogonality between the bases, we study the inner product matrix of the hybrid basis induced by the mass $M$ on the shape in \cref{fig:inner_product_matrix}. The matrix exihibits block form and is defined as the mass matrix in the reduced hybrid basis:
\begin{equation}
\label{eq:hybrid_mass_matrix}
M_{k_{\text{total}}}=\left[\begin{array}{cc}
I &  M^{12}\\
M^{21} & M_{k_{\text{Elas}}}
\end{array}\right]
\end{equation}\\
The top-left block clearly depicts an identity matrix as we know the LBO eigenfunctions are orthogonal. 
The bottom-right block corresponds to the spectral mass matrix $M_{k_{\text{Elas}}}$ for the elastic basis \cite{hartwig_elastic_2023}.
Deviation from the identity matrix is expected as the elastic basis is non-orthogonal. 
Inter-basis regions ($M^{12}$ and $M^{21}$) have non-zero entries, indicating e.g. the first 30 LB bases and the first 30 Elastic bases are mutually non-orthogonal. 
However, we observe that the cross region blocks are sparse. 
Approximately 90\%  of the cross region blocks have values below 0.1.

\subsection{Starting Assumptions}
In ~\cref{sec:hybrid_basis}, we make the assumption that the off-diagonal blocks of both the hybrid functional map and reduced hybrid mass matrix $M_{k_{\text{total}}}$ contain no off-diagonal blocks to separate the hybrid Fmap optimization. 
To see that such assumptions are plausible, we conduct an experiment on how well a functional map can represent an underlying ground truth correspondence by recovering the point-to-point map from a hybrid functional map via standard nearest neighbor search.
This experiment is carried out both without any assumptions (Naive) and with the assumptions of zero off-diagonal blocks (Ours). 
Results can be seen in \cref{fig:sup_zero_cross}. 
As can be observed, even though the optimal naive map exhibits "leakage",with mild assumptions and less parameters one can solve for a block-diagonal map representing the same level of accuracy, outperforming the baseline LB and Elastic Basis in overall detail alignment and coarse global alignments.

\begin{figure*}[t!]
  \centering
  \vspace{-0.9cm}
  \includegraphics[width=\textwidth]{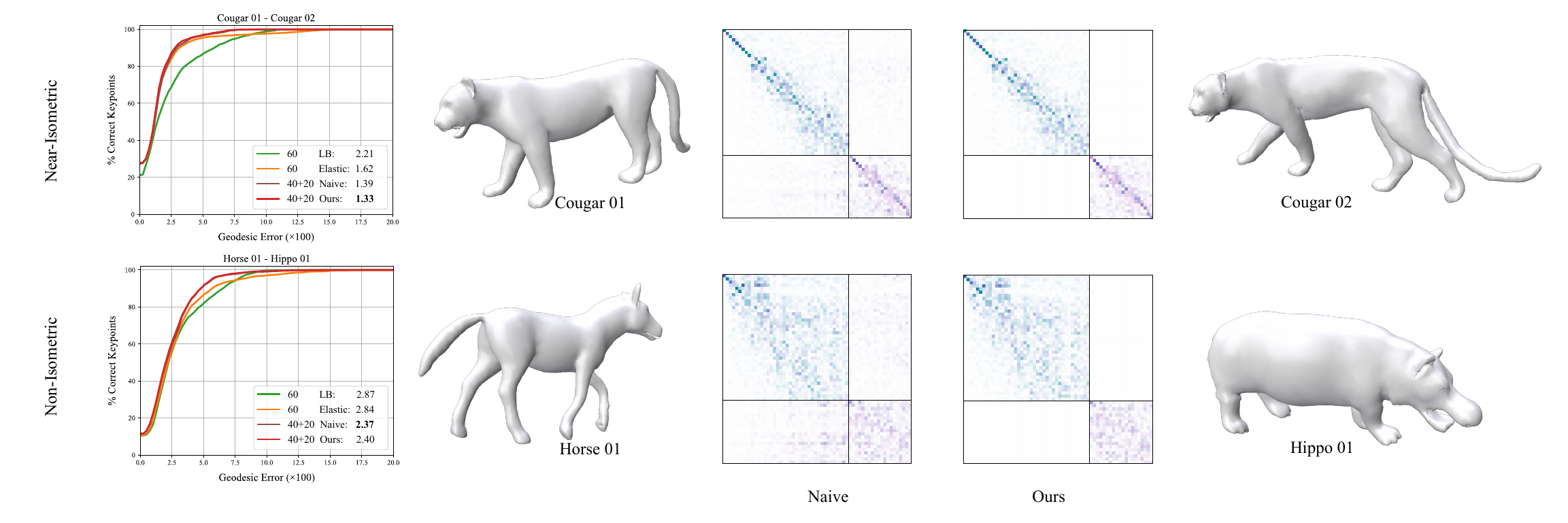}
  \vspace{-0.5cm}
  \caption{Point-to-point Map recovery accuracy from ground truth hybrid functional maps without any assumptions (Naive) and with zero cross region assumptions (Ours) between two representative pairs of shapes from the SMAL dataset, showcasing both near-isometric and non-isometric scenarios. We use $k=40+20$ as the spectral resolution.}
  \label{fig:sup_zero_cross}
\end{figure*}

\subsection{Optimization Block Matrix Formulation}
As demonstrate in \cref{fig:sup_zero_cross}, we observe that inter-basis matchings in a hybrid map do not improve performance.
In addition, such a map is harder to regularize (\cref{sec:block_suppl_ablation_exp}) and more expensive to compute.
We, therefore, impose the constraint $C^{21} = C^{12} = 0$ and show that this is equivalent to solving the optimization problems in~\cref{eq:linear_lsq} separately.

\begin{theorem}
\label{thm:seperable_block_matrix}
Let the off-diagonal blocks in the hybrid functional map ~\cref{eq:block_matrix_formulation} and the mass matrix in the reduced hybrid basis ~\cref{eq:hybrid_mass_matrix} be zero, i.e., there are no inter-basis matchings, and the bases are considered as if mutually orthogonal.
In hybrid function space, the energy in \cref{eq:linear_lsq} can then be equivalently formulated as two separate optimization problems:
\begin{align*}
C^{11}_* = \arg \min_C &E_{\text{LB}}(C) \quad \quad C^{22}_* = \arg \min_C E_{\text{Elas}}(C) \\
&C_* = \begin{pmatrix}
C^{11}_* & \mathbf{0} \\
\mathbf{0} & C^{22}_*
\end{pmatrix}
\end{align*}
\end{theorem}

\begin{proof}[Proof of ~\cref{thm:seperable_block_matrix}]
In the following we use $k_{\text{total}}$ as the total basis size, $k_{\text{LB}}$ as LB basis size, $k_{\text{Elas}}$ as the non-orthogonal elastic basis size.
The block matrix representation of the functional map \(C\) in the hybrid vector space is given by

\begin{equation*}
C = \begin{pmatrix}
C^{11} & \mathbf{0} \\
\mathbf{0} & C^{22}
\end{pmatrix},
\end{equation*}

where \(C^{11}\) and \(C^{22}\) represent the functional maps within the same basis types, and off diagonal blocks are fixed to zero as per our starting assumptions.

By our second assumption, the spectral mass matrix in the reduced hybrid basis assumes zero off diagonal blocks and is similarly given by\\

\begin{equation*}
M_{k_{\text{total}}}=\left[\begin{array}{cc}
I &  \mathbf{0}\\
\mathbf{0} & M_{k_{\text{Elas}}}
\end{array}\right]
\end{equation*}\\

We denote the hybrid basis as \( \Theta_i := [\Phi_i \, \Psi_i] \). 
The orthogonal projector operator \cite{hartwig_elastic_2023} of the hybrid basis, is given by \( \Theta_i^\dagger := M_{k_{\text{total},i}}^{-1} \Theta_i^\top M_i \).
Therefore the orthogonal projector can be written to block form:

\begin{align*}
\Theta_i^{\dagger} &=  \left[\begin{array}{c}
 \Phi_i^{\top} M_i \\
 M_{k_{\text{Elas},i}}^{-1} \Psi_i^{\top} M_i
 \end{array}\right] = \left[\begin{array}{l}
\Phi_i^{\dagger} \\
\Psi_i^{\dagger}
\end{array}\right]
\end{align*}

Then we have for the descriptor projecting to the hybrid basis as:
\begin{equation*}
    D_{\Theta_i} := \Theta_i^\dagger D_i = 
    \begin{bmatrix}
    \Phi_i^{\dagger} D_i \\
    \Psi_i^{\dagger} D_i
    \end{bmatrix}=
    \begin{bmatrix}
     D_{\Phi_i} \\
    D_{\Psi_i}
    \end{bmatrix}
\end{equation*}

We let $\Sigma_i := diag( \lambda_1, ... \lambda_{k_{LB}},  \gamma_1, ..., \gamma_k)$ the diagonal matrix of combined eigenvalues from $\Delta_i$ and $\text{Hess}\mathcal{W}_S[\text{Id}]$, respectively. Then, both the data and regularization terms in~\cref{eq:linear_lsq} can be expanded:

\begin{equation*}
E(C) = \| C D_{\Theta_1} - D_{\Theta_2} \|_{M_{k_{\text{total,2}}}}^2 + \lambda \|C \Sigma_1 - \Sigma_2 C \|_{HS}^2 \nonumber
\end{equation*}
\begin{figure*}[t!]
  \centering
  \vspace{-0.9cm}
  \includegraphics[width=0.9\textwidth]{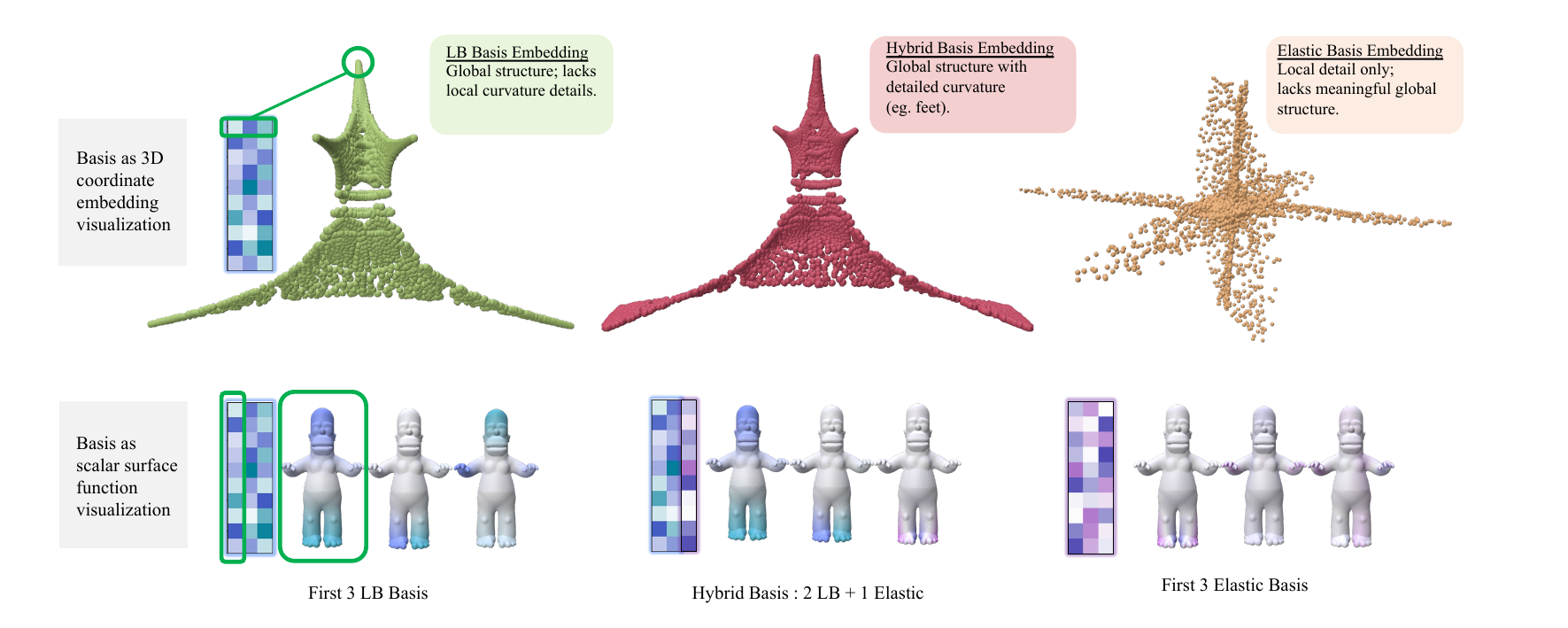}
  \vspace{-0.5cm}
  \caption{Characterization of different basis functions by visualizing them as scalar surface functions (bottom), and as 3d coordinate embeddings (top). While the elastic basis lacks global structure, it naturally extends the smooth approximation of the LBO basis to incorporate geometric details (creases) when hybridized.}
  \label{fig:sup_embedding}
\end{figure*}
\noindent We can express the data term in the block matrix format:
\begin{align*}
&E_{\text{data}}(C) \\
=& \left\| C D_{\Theta_1} - D_{\Theta_2} \right\|_{M_{k_{\text{total,2}}}}^2 \\
=& \left\| 
\begin{bmatrix}
C^{11}D_{\Phi_1} - D_{\Phi_2} \\
\sqrt{M_{k_{\text{Elas,2}}}}(C^{22}D_{\Psi_1} - D_{\Psi_2})
\end{bmatrix}
\right\|_F^2
\end{align*}

Due to the additivity of the Frobenius norm, the two terms can be minimized separately. A similar condition holds for the regularization term as:
\begin{align*}
&E_{\text{reg}}(C) \\
=& ||C\ \Sigma_1 - \Sigma_2 \ C||^2_{\text{HS}} \\
=& \left\| 
\begin{bmatrix}
C^{11}\Lambda_1 - \Lambda_2 C^{11} \\
\sqrt{M_{k_{\text{Elas,2}}}} (C^{22}\Gamma_1 - \Gamma_2 C^{22}) \sqrt{M_{k_{\text{Elas,1}}}^{-1}}
\end{bmatrix}
\right\|_F^2
\end{align*}

Now, the optimization problem decouples into two separate problems, one for each basis type. These can be solved independently to obtain the optimal functional maps $C^*_{11}$ and $C^*_{22}$ within the LBO and elastic bases, respectively.

\begin{align*}
C^{11}_* &= \arg\min_{C^{11}} E_{\text{LB}}(C^{11}) \\
C^{22}_* &= \arg\min_{C^{22}} E_{\text{Elas}}(C^{22})
\end{align*}

\end{proof}

Estimating a hybrid functional map in this manner is an effective regularization which aids computational efficiency, however the two bases are not separated everywhere. 
The notion of "hybrid" is conceptually important and essential for the rest of optimization stages, including point map conversion and the training of neural network feature extractors as seen in \cref{sec:suppl_ablation}.
\subsection{Point-to-Point Map Conversion}

In order to obtain a point-to-point map $P$ from a general functional map $C$ between non-orthogonal bases, the following minimization objective is considered~\cite{hartwig_elastic_2023}:
\begin{align*}
\min_{P \in \{0,1\}^{m \times n}}& \|C M_{k,1}^{-1} \Psi_{1}^{\top} - M_{k,2}^{-1} \Psi_{2}^{\top} P^{\top}\|_{M_{k,2}}^2 \\
&\text{s.t. } P^T \mathbf{1} = \mathbf{1}
\end{align*}
This formulation can be solved efficiently by considering the nearest neighbor in $\sqrt{M_{k,2}} C M_{k,1}^{-1} \Psi_{1}^{\top}$ for every point in $\sqrt{M_{k,2}^{-1}} \Psi_{2}^{\top}$. 
For a hybrid functional map, the above objective cannot be decoupled, and the search yields a single point-to-point map based on the embeddings in the hybrid space, as illustrated in~\cref{fig:sup_embedding}. 
Note that the block structure assumption can still be taken advantage of for efficient matrix multiplications prior to nearest neighbor search.

%% file: content/near_isometric.tex
\begin{table*}[h!t!]

\small
\setlength{\tabcolsep}{5pt}
    \centering
    \small
    
\caption{\textbf{Additional Experimental Results of Hybrid ULRSSM compared to ULRSSM ~\cite{cao_unsupervised_2023} in cross-dataset near-isometric scenario on FAUST, SCAPE, and SHREC'19 datasets.} Our hybrid adaptation generalizes across datasets and also shows improved performance under such settings. ULRSSM results are reproduced from the official repository. }

        \begin{tabular}{@{}lccccccccc@{}}
        \toprule
        \multicolumn{1}{l}{Train}  & \multicolumn{3}{c}{\textbf{FAUST}}   & \multicolumn{3}{c}{\textbf{SCAPE}}  & \multicolumn{3}{c}{\textbf{FAUST + SCAPE}} \\ \cmidrule(lr){2-4} \cmidrule(lr){5-7} \cmidrule(lr){8-10}
        \multicolumn{1}{l}{Test} & \multicolumn{1}{c}{\textbf{FAUST}} & \multicolumn{1}{c}{\textbf{SCAPE}} & \multicolumn{1}{c}{\textbf{SHREC'19}} & \multicolumn{1}{c}{\textbf{FAUST}} & \multicolumn{1}{c}{\textbf{SCAPE}} & \multicolumn{1}{c}{\textbf{SHREC'19}} & \multicolumn{1}{c}{\textbf{FAUST}} & \multicolumn{1}{c}{\textbf{SCAPE}} & \multicolumn{1}{c}{\textbf{SHREC'19}}
        \\ \midrule
        
        \multicolumn{1}{l}{ULRSSM ~\cite{cao_unsupervised_2023}} & \multicolumn{1}{c}{1.6}  & \multicolumn{1}{c}{2.2} & \multicolumn{1}{c}{6.7}  & \multicolumn{1}{c}{1.6} & \multicolumn{1}{c}{1.9} & \multicolumn{1}{c}{5.7} & \multicolumn{1}{c}{1.6}    & \multicolumn{1}{c}{2.1} & \multicolumn{1}{c}{4.6} \\
        \multicolumn{1}{l}{Hybrid ULRSSM (Ours)} & \multicolumn{1}{c}{\textbf{1.5}}  & \multicolumn{1}{c}{\textbf{2.1}} & \multicolumn{1}{c}{\textbf{5.5}}  & \multicolumn{1}{c}{\textbf{1.5}} & \multicolumn{1}{c}{\textbf{1.8}} & \multicolumn{1}{c}{\textbf{5.4}} & \multicolumn{1}{c}{\textbf{1.5}}    & \multicolumn{1}{c}{\textbf{2.0}} & \multicolumn{1}{c}{\textbf{3.6}} \\
        \bottomrule
        \end{tabular} 
\label{tab:near-iso}
\end{table*}